\documentclass{article}

\usepackage{arxiv}

\usepackage{amsmath,amsfonts,amsthm}
\usepackage{algorithmic}
\usepackage{algorithm}
\usepackage{array}
\usepackage[caption=false,font=normalsize,labelfont=sf,textfont=sf]{subfig}
\usepackage{textcomp}
\usepackage{stfloats}
\usepackage{url}
\usepackage{verbatim}
\usepackage{graphicx}
\usepackage[pagebackref=true]{hyperref}

\usepackage{tikz}
\usepackage{xypic}
\usepackage{bm}
\usepackage{mathtools}
\providecommand{\discreteLift}[1]{\Phi({#1})}

\providecommand{\R}{\mathbb{R}}
\providecommand{\SO}{\mathbf{SO}}

\providecommand{\SE}{\mathbf{SE}}
\providecommand{\grpG}{\mathbf{G}}

\providecommand{\gothg}{\mathfrak{g}}

\providecommand{\gothX}{\mathfrak{X}} % as in X(M)

\providecommand{\calM}{\mathcal{M}}
\providecommand{\calN}{\mathcal{N}}

\providecommand{\vecL}{\mathbb{L}}

\providecommand{\Sym}{\mathbb{S}} % symmetric matrix $\Sym(n)$
\providecommand{\tT}{\mathrm{T}} % tangent objects eg $\tT \calM$
\providecommand{\GP}{\mathbf{N}} % Gaussian noise process.
\providecommand{\tL}{\mathrm{L}} % left multiplication
\providecommand{\tR}{\mathrm{R}} % left multiplication

\DeclareMathOperator{\Ad}{Ad}
\DeclareMathOperator{\ad}{ad}

\providecommand{\id}{\mathrm{id}} % identity map
\providecommand{\Lyap}{\mathcal{L}} %% aggregate cost

\providecommand{\td}{\mathrm{d}}
\providecommand{\tD}{\mathrm{D}}
\providecommand{\tL}{\mathrm{L}}

\providecommand{\ddt}{\frac{\td}{\td t}}

\providecommand{\dt}{\td t}

\providecommand{\mr}[1]{\mathring{#1}} % reference element.

\usepackage{accents}
\usepackage{mathtools}
\makeatletter
\providecommand{\scirc}{%
    \hbox{\fontfamily{\rmdefault}\fontsize{0.4\dimexpr(\f@size pt)}{0}\selectfont{\raisebox{-0.52ex}[0ex][-0.52ex]{$\circ$}}}}
\makeatother
\makeatletter
\providecommand{\ucirc}{%
    \hbox{\fontfamily{\rmdefault}\fontsize{0.4\dimexpr(\f@size pt)}{0}\selectfont{\raisebox{0.0ex}[0ex][-0.52ex]{$\circ$}}}}

\makeatother

\mathchardef\mhyphen="2D
\providecommand{\etal}{\textit{et al.~}}

\usepackage{xcolor}

\usepackage{amssymb}
\usepackage{amsmath}
\usepackage{mathdots}

\usepackage{mathtools}

\usepackage{tensor}

\usepackage{urwchancal}
\DeclareFontFamily{OT1}{pzc}{}
\DeclareFontShape{OT1}{pzc}{m}{it}{<-> s * [1.10] pzcmi7t}{}
\DeclareMathAlphabet{\mathpzc}{OT1}{pzc}{m}{it}

\usepackage{graphicx}
\usepackage{ifpdf}
\ifpdf
\usepackage{epstopdf}
\PrependGraphicsExtensions{.svg}
\fi
\newtheorem{theorem}{Theorem}[section]
\newtheorem{lemma}[theorem]{Lemma}

\newtheorem{definition}[theorem]{Definition}

\newtheorem{remark}[theorem]{Remark}

\providecommand{\logG}{\log_\grpG}
\providecommand{\logGv}{\log^\vee_\grpG}
\providecommand{\expG}{\exp_\grpG}

\newcommand{\ddz}[1]{\left. \frac{\td}{\td #1} \right|_{#1=0}}

\newcommand{\publicationdetails}{Submitted to Control Engineering Practice}
\newcommand{\publicationversion}{Preprint}

\title{\LARGE \bf The Difference between the Left and Right Invariant Extended Kalman Filter}

\author{
    \href{https://orcid.org/0000-0001-7969-7039}{\includegraphics[scale=0.06]{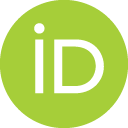}\hspace{1mm}
Yixiao Ge$^*$}
\\
    Systems Theory and Robotics Group \\
    School of Engineering \\
	Australian National University \\
    ACT, 2601, Australia \\
    \texttt{Yixiao.Ge@anu.edu.au} \\
\And    \href{https://orcid.org/0009-0007-3918-9196}{\includegraphics[scale=0.06]{orcid.png}\hspace{1mm}
Giulio Delama}
\\
    Control of Networked Systems Group \\
	University of Klagenfurt \\
    Klagenfurt, Austria \\
    \texttt{giulio.delama@aau.at} \\
\And    \href{https://orcid.org/0000-0002-3415-7378}{\includegraphics[scale=0.06]{orcid.png}\hspace{1mm}
Martin Scheiber}
\\
    Control of Networked Systems Group \\
	University of Klagenfurt \\
    Klagenfurt, Austria \\
    \texttt{martin.scheiber@aau.at} \\
\And    \href{https://orcid.org/0000-0002-1774-3236}{\includegraphics[scale=0.06]{orcid.png}\hspace{1mm}
Alessandro Fornasier}
\\
    Hexagon Robotics \\
	Zürich, Switzerland \\
    \texttt{alessandrofornasierphd@gmail.com} \\
\And    \href{https://orcid.org/0000-0003-4391-7014}{\includegraphics[scale=0.06]{orcid.png}\hspace{1mm}
Pieter van Goor}
\\
    Robotics and Mechatronics (RaM) Group \\
    EEMCS Faculty \\
	University of Twente \\
    Enschede, The Netherlands \\
    \texttt{p.c.h.vangoor@utwente.nl} \\
\And    \href{https://orcid.org/0000-0001-6906-5409}{\includegraphics[scale=0.06]{orcid.png}\hspace{1mm}
Stephan Weiss}
\\
    Control of Networked Systems Group \\
	University of Klagenfurt \\
    Klagenfurt, Austria \\
    \texttt{stephan.weiss@aau.at} \\
\And	\href{https://orcid.org/0000-0002-7803-2868}{\includegraphics[scale=0.06]{orcid.png}\hspace{1mm}
    Robert Mahony}
\\
    Systems Theory and Robotics Group \\
    School of Engineering \\
	Australian National University \\
    ACT, 2601, Australia \\
	\texttt{Robert.Mahony@anu.edu.au} \\
}

\begin{document}
\maketitle

\begin{abstract}
The extended Kalman filter (EKF) has been the industry standard for state estimation problems over the past sixty years.
The invariant extended Kalman filter (IEKF) \cite{barrauInvariantExtendedKalman2017} is a recent development of the EKF for the class of group-affine systems on Lie groups that has shown superior performance for inertial navigation problems. 
The IEKF comes in two versions, left- and right- handed respectively, and there is a perception in the robotics community that these filters are different and one should choose the handedness of the IEKF to match handedness of the measurement model for a given filtering problem.
In this paper, we revisit these algorithms and demonstrate that the left- and right- IEKF algorithms (with reset step) are identical, that is, the choice of the handedness does not affect the IEKF's performance when the reset step is properly implemented.
The reset step was not originally proposed as part of the IEKF, however, we provide simulations to show that the reset step improves asymptotic performance of all versions of the filter, and should be included in all high performance algorithms. 
The GNSS-aided inertial navigation system (INS) is used as a motivating example to demonstrate the equivalence of the two filters.
\end{abstract}

%%%%%%%%%%%%%%%%%%%%%%%%%%%%%%%%%%%%%%%%%%%%%%%%%%%
\section{Introduction}
\label{sec:introduction}

The extended Kalman filter (EKF) has been the industry standard nonlinear state estimation algorithm for the past sixty years \cite{maybeckStochasticModelsEstimation1982,barfootStateEstimationRobotics2024}.
The original formulation of the EKF was developed for systems evolving on global Euclidean spaces \cite{kalmanNewApproachLinear1960}, however, the first application of EKF, where it was applied to attitude estimation problem in the Apollo mission \cite{smithApplicationStatisticalFilter1962}, involved a system evolving on the special orthogonal group $\SO(3)$. 
The advent of uncrewed aerial vehicles (UAVs) led to increased research focus on the attitude filtering problem \cite{thienelCoupledNonlinearSpacecraft2003,mahonyNonlinearComplementaryFilters2008}.
This in turn led to a surge in the development of filtering algorithms for systems evolving on general Lie groups and homogeneous spaces.
Motivated by the attitude estimation problem and the more general question of building inertial navigation systems, Bonnabel \etal proposed the Invariant Extended Kalman Filter (IEKF), a general filtering methodology for systems on Lie groups, in a series of works \cite{bonnabelLeftinvariantExtendedKalman2007,bonnableInvariantExtendedKalman2009}.
In \cite{wangErrorPropagationEuclidean2006,long2013banana}, Chirikjian \etal showed that for left invariant kinematics on Lie groups, the error linearisation in the propagation stage is global, an important foundation for high performance filtering algorithms.
In \cite{barrauInvariantExtendedKalman2017}, Barrau and Bonnabel identified a class of `group affine' systems, for which they showed that the IEKF provides global error linearisation in the predict step.
Indeed, if the system admits a group-affine structure with left- (resp. right-) invariant inputs and right- (resp. left) equivariant outputs, then under a specific choice of covariance gains the convergence rate and basin of attraction of the IEKF becomes \emph{trajectory-independent} \cite{barrauInvariantKalmanFiltering2018}.
In a parallel research thread, Mahony \etal \cite{mahonyObserverDesignNonlinear2022,vangoorEquivariantFilterEqF2023,geEquivariantFilterDesign2022} proposed the equivariant filter (EqF), a general filter design methodology for systems evolving on homogeneous spaces, which specialises to the IEKF when the symmetry Lie groups and the state space are the same.

The IEKF is an error-state EKF that is derived by linearising the error dynamics, where the state estimation error can be defined using the \emph{left-} or \emph{right-invariant} error defined on a Lie group. 
The choice of the handedness of the error state leads to two different formulations of the IEKF, the left (L-IEKF) and the right (R-IEKF) filters.
There has been a long-standing argument in the literature that the choice of the handedness of the IEKF should be made to match the measurement model \cite{barrauInvariantExtendedKalman2017}. 
Left- versus right- observation models are also sometimes referred to as the \emph{spatial} or \emph{global} measurements versus \emph{ego} or \emph{local} measurements \cite{hartleyContactaidedInvariantExtended2020}. 
Typical spatial measurement models include global navigation satellite system (GNSS) measurements, while ego measurements are provided by body frame sensors such as inertial measurement units (IMUs), wheel odometers and cameras \cite{hanCovarianceSwitchBasedInvariant2024}.

In this paper, we use the concept of concentrated Gaussian distributions \cite{chirikjianStochasticModelsInformation2011,geGeometricPerspectiveFusing2024} to derive the left- and right- IEKFs with reset.
We go on to show that the left- and right- IEKF algorithms (including reset step) are stochastically equivalent; that is, they generate identical updates of the concentrated Gaussian parameters.
The conclusion is that the choice to implement an IEKF (including reset step) with either left- or right-invariant error is arbitrary.
Nevertheless, we go on to provide some analysis of the potential differences.
Firstly, discretising the continuous-time filter equations in the left- or right- representations introduces differences between the filters and choosing the representation for which the error dynamics are closest to linear will yield some advantage.
However, if the system function is expressed explicitly as a discrete propagation, then equivalence is recovered (\ref{app:discrete_equivalence}).
Secondly, choosing not to implement the reset step also breaks the equivalence.
We show empirically that the original IEKF with matched handedness appears to benefit slightly during the transient phase, at least for the INS problem considered.
However, we also demonstrate empirically that the reset step appears to always benefit the asymptotic phase of the filter in all algorithms, a result that aligns with prior work in the literature \cite{muellerCovarianceCorrectionStep2017}.
We have chosen to focus on the popular imperfect IEKF \cite{barrauInvariantExtendedKalman2017} for which the question of left- versus right- implementations is most contested in the literature.
However, the main results of the paper apply to all of the filters discussed in the recent paper \cite{fornasierEquivariantSymmetriesInertial2025}, including the IEKF \cite{barrauInvariantExtendedKalman2017} and the TFG-IEKF \cite{barrauGeometryNavigationProblems2023}.
Furthermore, since the filter is implemented on the lifted system on the Lie group, the results also apply to the more general Equivariant Filters \cite{vangoorEquivariantFilterEqF2023}.
This paper is complementary to the very recent parallel work by Maurer \etal \cite{maurerEquivalenceLeftRightInvariant2025}, which contains some closely related equivalence proofs.
In this paper, we go further in analysing discrete-time systems, considering equivariant outputs, and investigating the role of the reset step in the L-IEKF and R-IEKF.

The IEKF was originally introduced without the reset step,  while the main result of this paper depends on inclusion of the reset step. 
The authors stress that we are not taking a position on whether the reset step should be used in all applications or scenarios. 
Rather we wish to contribute to the present debate in the literature with theoretical analysis and some experimental studies. 
We believe there are scenarios where running the filter without reset is justified, and scenarios where there is significant benefit from the reset. 
Our only strong claim is that, if the reset is included, then there is no difference between a left- and right-invariant IEKF, caveat discretisation error.

\section{Preliminaries}
\label{sec:preliminary}

Let $\calM$ be a smooth manifold with dimension $m$.
The tangent space at a point $\xi\in\calM$ is denoted $\mathrm{T}_\xi\calM$.
The tangent bundle is denoted $\tT\calM$.
Given a differentiable function between smooth manifolds $h:\calM\rightarrow\calN$, its derivative at $\xi^\circ$ is written as
\begin{align*}
    \tD_\xi|_{\xi^\circ}h(\xi): \tT_{\xi^\circ}\calM\rightarrow \tT_{h(\xi^\circ)}\calN.
\end{align*}
The notation $\tD h:\tT\calM\rightarrow \tT\calN$ denotes the differential of $h$ with an implicit base point.

Let $\grpG$ be a general Lie group with dimension $n$, associated with the Lie algebra $\gothg$.
Let $\id$ denote the identity element of $\grpG$.
Given arbitrary $X,Y\in\grpG$, the left and right translations are denoted by $\textrm{L}_X, \textrm{R}_X : \grpG \to \grpG$, and are defined by
\[
    \textrm{L}_X(Y):=XY, \quad \textrm{R}_X(Y):=YX.
\]

The Lie algebra $\gothg$ is isomorphic to a vector space $\R^n$ with the same dimension.
We use wedge $(\cdot)^\wedge:\R^n\rightarrow\gothg$ and vee $(\cdot)^\vee:\gothg\rightarrow\R^n$ operators to map between the Lie algebra and the vector space.
The Adjoint map for the group $\grpG$, $\Ad_X:{\gothg}\to{\gothg}$ is defined by
\[
    \Ad_{X}[{{u}^{\wedge}}] = \tD \textrm{L}_{X} \circ\tD \textrm{R}_{X^{-1}}\left[{u}^{\wedge}\right] ,
\]
for every $X \in \grpG$ and ${{u}^{\wedge} \in \gothg}$, where $\tD \textrm{L}_{X}$, and $\tD \textrm{R}_{X}$ denote the differentials of the left and right translations, respectively.
The adjoint map for the Lie algebra $\ad_{{u}^\wedge}: {\gothg}\to{\gothg}$ is given by
\begin{equation*}
    \ad_{{u}^\wedge}{{v}^{\wedge}} = \left[{u}^{\wedge}, {v}^{\wedge}\right] ,
\end{equation*}
and is equivalent to the Lie bracket.
Given particular wedge and vee maps, the Adjoint matrix $\Ad_X^\vee  \in \R^{n\times n}$ and adjoint matrix $\ad_u^\vee  \in \R^{n\times n}$ are defined by
\begin{align*}
    \Ad_{X}^\vee u &= \left(\Ad_{X}{{u}^{\wedge}}\right)^{\vee}, \\
    \ad_u^\vee v &= \left(u^\wedge v^\wedge - v^\wedge u^\wedge\right)^\vee = \left[{u}^{\wedge}, {v}^{\wedge}\right]^\vee.
\end{align*}

Let $\expG:\gothg\to\grpG$ denote the exponential map from the Lie algebra to the group.
For matrix Lie groups such as $\SO(3),\SE(3)$, this map is simply the matrix exponential.
Let $\grpG'\subset\grpG$ be the subset of $\grpG$ where the exponential map is invertible, then one can define the logarithm map $\logG:\grpG'\to\gothg$ and $\logGv:\grpG'\to\R^n$.

\section{Problem Description}
\label{sec:problem_description}

Consider a system evolving on an $m$-dimensional Lie group $\grpG$.
The input space $\vecL$ is a $\ell$-dimensional linear space.
The continuous-time noise-free system dynamics are given by  
\begin{align}
    \dot{X} &= f_v(X) = X \Lambda(X, v), \label{eq:dynamicsCT}
\end{align}
where $X\in\grpG$ denotes the state of the system, and $f: \vecL \to \gothX(\grpG)$,
$v \mapsto f_v \in\gothX(\grpG)$ is the system function with left trivialisation $\Lambda : \grpG \times \vecL \to \gothg$ given by $\Lambda(X,v) := X^{-1} f_v(X)$.
For example, if the system has a spatial velocity $W \in \gothg$ such that
\[
\dot{X} = W X ,
\]
then the model that we use is 
\[
\dot{X} = X \Ad_{X^{-1}} W = X \Lambda(X,W). 
\]

\begin{remark}
    While the left-handed trivialisation is the standard representation in the literature, it is, of course, possible to redo all the theory in right-handed trivialisation.  
    However, since most robotics systems admit both spatial and body velocities, such as the INS problem discussed in Sec. \ref{sec:experiment}, there is no simplification in the formulas obtained in choosing one handedness over the other. 
\end{remark}

In this work, we assume that the noise process entering the system dynamics is left-invariant, that is, the noise process can be written as a stochastic differential equation
\begin{align}
    \td{X} &= X \left(\Lambda(X, v) \dt + \Upsilon[Q^\frac{1}{2}\td w]\right) \label{eq:dynamics_noise}
\end{align}
where $\td w$ is a Brownian motion in $\R^\ell$, $\Upsilon[\cdot]:\R^\ell \to \gothg$ is a constant linear operator, and $Q \in \Sym_+(\ell)$ is a positive definite matrix.
This model corresponds to the models considered recently in the literature \cite{yeUncertaintyPropagationUnimodular2024,maurerEquivalenceLeftRightInvariant2025}. 
Understanding the noise model in invariant filtering is still an active topic \cite{mahonyObserverDesignNonlinear2022}, however, the authors believe that sensible alternative choices of noise models will not change the equivalence proof, and the assumption made here is sufficiently general to cover most of the cases in practice.

The configuration output is a discrete measurement process given by the output function $h: \grpG\to\R^n$
\begin{align}
    y(t_k) &= h(X(t_k)) + \nu_k & \nu_k \sim \GP(0, R_k), \label{eq:measurement_generic}
\end{align}
at times $\{t_1, \ldots, t_k, \ldots\}$ and where $\nu_k$ is a zero-mean white noise with covariance $R_k \in \Sym_+(n)$.

Systems of this form given by \eqref{eq:dynamicsCT} and \eqref{eq:measurement_generic} are termed \emph{hybrid} systems and provide a good model for filter design for most robotic systems with fast input measurements and slow configuration measurements. 
We also consider fully discrete-time implementations of the IEKF in \ref{app:discrete_equivalence}.

\section{Left and Right IEKF Definition}
\label{sec:IEKF_definition}
In this section, we present a formal definition of the L-IEKF and R-IEKF with reset.
We start with the probability distributions associated with the two filters and derive the IEKFs in terms of the predict, update and reset steps for Kalman filters. 

\subsection{Concentrated Gaussian Distribution}
\label{sec:CGD}
The underlying goal of a stochastic filter is to estimate the probability density function of the system state given the measured inputs and outputs.
We use the concept of extended concentrated Gaussian distribution \cite{wangErrorPropagationEuclidean2006,geEquivariantFilterDesign2022,geGeometricPerspectiveFusing2024} to model the information state of the system.
Note that concentrated Gaussians depend on handedness, that is, they can be defined using left multiplication or right multiplication. 
The two representations can be shown to be equivalent through the $\Ad$ operation.

For an $m$-dimensional random variable $g\in\grpG$, the left-concentrated Gaussian distribution (L-CGD) is defined as 
\begin{align}\label{eq:L_CGD}
    p^\tL(g ; x, \mu, \Sigma) = \alpha \exp \left\{-\frac{1}{2}\lvert\logGv\left(x^{-1}g\right)-\mu^\vee\rvert_{ \Sigma^{-1}}\right\},
\end{align}
where $x\in\grpG$ is termed the \emph{reference point}, $\mu\in\gothg$ is termed the \emph{mean} and $\alpha\in\R$ is the normalising factor.
The \emph{covariance} $\Sigma\in\Sym_+(m)$ is an $m\times m$ symmetric and positive definite matrix.
This is equivalent to defining a random variable
\[
g = x\expG(\epsilon), \qquad \epsilon\sim \GP(\mu^\vee, \Sigma),
\]
where $\GP$ denotes the Gaussian process on $\R^m$.

Similarly, the right-concentrated Gaussian distribution (R-CGD) is defined as
\begin{align}\label{eq:R_CGD}
    p^\tR(g ; x, \mu, \Sigma) = \alpha \exp \left\{-\frac{1}{2}\lvert\logGv\left(gx^{-1}\right)-\mu^\vee\rvert_{\Sigma^{-1}}\right\},
\end{align}
which is equivalent to the random variable definition
\[  
g = \expG(\epsilon)x, \qquad \epsilon\sim \GP(\mu^\vee, \Sigma).
\]
For the rest of this paper, we will use $\GP^\tL_{x_\tL}(\mu_\tL, \Sigma_\tL)$ and $\GP^\tR_{x_\tR}(\mu_\tR, \Sigma_\tR)$ to denote the L-CGD and R-CGD, respectively.
The log-likelihood of a L-CGD is defined by
\begin{align*}
    \Lyap_\tL(\mu_\tL, x_\tL, \Sigma_\tL) := \frac{1}{2} \vert \log_\grpG(x_\tL^{-1} g) - \mu_\tL \vert_{\Sigma_\tL^{-1}}^2.
\end{align*}
where we treat the scaling factor as a cumulant offset and ignore it.
The log-likelihood of a R-CGD is 
\begin{align*}
    \Lyap_\tR(\mu_\tR, x_\tR, \Sigma_\tR) := \frac{1}{2} \vert \logG(g x_\tR^{-1}) - \mu_\tR \vert_{\Sigma_\tR^{-1}}^2.
\end{align*}

The information state of the L-IEKF with L-CGD for state representation is finite dimensional and is parameterised by a triple $(\mu_\tL, x_\tL, \Sigma_\tL) \in \gothg \times \grpG \times \Sym_+(m)$, representing the offset (or mean), the reference state (or group mean), and the covariance, respectively.
When the offset $\mu = 0$ then the reference is the group mean as understood by a classical concentrated Gaussian \cite{wangErrorPropagationEuclidean2006}. 
The R-IEKF state is similarly composed of $(\mu_\tR, x_\tR, \Sigma_\tR) \in \gothg \times \grpG \times \Sym_+(m)$. 

If $x_\tR = x_\tL$ then one has 
\begin{align*}
\expG(\varepsilon_\tR)x_\tR & = x_\tR x_\tR^{-1} \expG(\varepsilon_\tR)x_\tR \\
& = x_\tL \expG(\Ad_{x_\tR^{-1}} \varepsilon_\tR). 
\end{align*}
It follows that if $x_\tR = x_\tL$ and 
\begin{align*}
\varepsilon_\tL = \Ad_{x_\tR^{-1}} \varepsilon_\tR \sim \GP(\Ad_{x_\tR^{-1}}\mu_\tR, 
\Ad_{x_\tR^{-1}}^\vee \Sigma_\tR {\Ad_{x_\tR^{-1}}^\vee}^\top)
\end{align*}
then the information state captured by the two filters is identical. 
This property is used to prove the equivalence of the two filters in the following sections  (see also \cite{hanCovarianceSwitchBasedInvariant2024}).
 
The underlying integration measure associated with the definition of left- and right- concentrated Gaussians is implicit in the Euclidean coordinates placed on the Lie algebra. 
In the case where the Lie group is unimodular, it is possible to define a bi-invariant Haar measure directly on the group \cite{chirikjianStochasticModelsInformation2011}. 
Choosing a different measure will subtly alter the global structure of the pdf on the group, however, the local structure of the pdf will not be affected and in particular, the second-order statistics (mean and covariance) expressed in the algebra coordinates will be unchanged.
Filters that compute integrals of the pdf (e.g. the particle filter, etc) will be affected by the choice of measure, however, an EKF depends only the local second-order statistics expressed in the tangent space and is invariant to affine transformation of these coordinates making it measure independent.

\begin{remark}
By EKF algorithms, we mean algorithms where the filter steps are based on linearisation of the underlying system. 
Algorithms such as the well-known unscented Kalman filter (UKF) \cite{julierNewApproachFiltering1995}, or where integral properties of the distributions are used \cite{yeUncertaintyPropagationUnimodular2024}, are not true EKFs in this sense. 
The authors note that the UKF has been shown to have advantages in certain situations and there is active research into other more sophisticated algorithms that may lead to new filters that outperform the more classical IEKF. 
The focus of the paper is only on proving the equivalence of the left- and right- version of the IEKF. 
\end{remark}

\subsection{IEKF definition}
In this section, we define the IEKF dynamics for a hybrid system given by \eqref{eq:dynamicsCT} and \eqref{eq:measurement_generic}, that is, with continuous propagation and discrete measurement update.
The meaning of the parameters for a left or right filter implementation will be interpreted in terms of the left or right concentrated Gaussian distributions, respectively. 

\subsubsection{Notation}
An IEKF is often stated as an algorithm with states $(\hat{X}, \Sigma)$, however, in order to properly capture the reset step inherent in the EKF formulation, we use the extended concentrated Gaussian to model the filter state, given by $(\mu, \hat{X}, \Sigma)$.
Considering the reset step explicitly requires notation that captures the state of the filter through the three stages of the filter: \emph{predict}, \emph{update} and \emph{reset}. 
We will use the following notation for the state of the L-IEKF:
\small
\[\xymatrix@=0.5cm{
    (\mu_\tL(t_k), \hat{X}_\tL(t_k), \Sigma_\tL(t_k)) \ar[d]^{\text{predict}} \\ 
    (\mu_\tL^-(t_{k+1}), \hat{X}_\tL^-(t_{k+1}), \Sigma_\tL^-(t_{k+1})) \ar[d]^{\text{update}} \\
    (\mu_\tL^+(t_{k+1}), \hat{X}_\tL^+(t_{k+1}), \Sigma_\tL^+(t_{k+1})) \ar[d]^{\text{reset}}\\
    (\mu_\tL(t_{k+1}), \hat{X}_\tL(t_{k+1}), \Sigma_\tL(t_{k+1}))
}\]\normalsize
and use analogous notation for the R-IEKF.
Here, the full state of the filter at the start of one filter iteration is referred to without superscripts. 
We use a superscript `-' to indicate prediction without update.   
The `+' indicates that the new information available at time $t_{k+1}$ has been used to update the state, but expressed in the old coordinates. 
Finally, the reset step returns the full filter state in the new coordinates. 
Since the whole discussion occurs for filter step $t_{k+1}$, we will suppress the time index when it will not cause confusion to make the notation more concise.
We will use dedicated notation in different steps of the algorithm even when the values may be the same or known to improve comprehension. 
For example, both $\mu_\tL$ and $\mu_\tL^-$ are zero, however, we will preserve the full notation to make the steps in the algorithm clear. 

\subsubsection{Error definition}
The propagation step of the filter is associated with how the infinitesimal error in the neighbourhood of the true trajectory propagates.
For a classical error-state EKF, the error used is the Euclidean difference in local coordinates. 
In the case of systems on Lie-groups, there is an intrinsic definition of error, termed the \emph{left-} or \emph{right-invariant} errors given by
\begin{align*}
    E_{\tL} &:= \hat{X}^{-1} X,  &\text{and}&
    & E_{\tR} := X \hat{X}^{-1}.
\end{align*}
respectively.
Note that it is also possible to study the inverses of these errors, namely $E_{\tL}^{-1} = X^{-1} \hat{X}$ and $E_{\tR}^{-1} = \hat{X} X^{-1}$, but it is well established that inverting the error does not change the filter.

Let $v_m$ denote the measured velocity, the error dynamics are given by
\begin{align*}
    \ddt E_{\tL} 
    & = -\Lambda(\hat{X}_\tL, v_m) \hat{X}_\tL^{-1} X + \hat{X}_\tL^{-1}X  \Lambda(X, v) \\
    &= -\Lambda(\hat{X}_\tL, v_m) E_\tL + E_\tL \Lambda(X, v),
    \\ 
    \ddt E_{\tR} 
    & = X \left( \Lambda(X, v) - \Lambda(\hat{X}_\tR, v_m) \right) \hat{X}_\tR^{-1} \\
    &= E_\tR \Ad_{\hat{X}_\tR} \left( \Lambda(X, v) - \Lambda(\hat{X}_\tR, v_m) \right). 
\end{align*}
Recalling \eqref{eq:dynamics_noise}, then for the case of left-invariant error one has 
\small
\begin{align}
\td E_{\tL} 
& = -\Lambda(\hat{X}_\tL, v_m) E_\tL \dt + E_\tL \Lambda(X, v)\dt \notag \\ 
& = 
- \left(  \Lambda(\hat{X},v)\dt + \Upsilon [Q^{\frac{1}{2}} \td w ] \right)E_\tL + E_\tL \Lambda(X, v)\dt \notag \\  
& = 
\left( - \Lambda(\hat{X},v) E_\tL  + E_\tL \Lambda(X, v) \right) \dt 
-\Upsilon [Q^{\frac{1}{2}} \td w ] E_\tL \notag \\
& = 
\left( - \Lambda(\hat{X},v) E_\tL  + E_\tL \Lambda(X, v) \right) \dt 
-\tD \tR_{E_\tL} \Upsilon [Q^{\frac{1}{2}}\td w]. 
\label{eq:left_handed_error_linearisation}
\end{align}
\normalsize
For the case of right-handed error one has \small
\begin{align}
\td E_{\tR} 
& = E_\tR \Ad_{\hat{X}_\tR} \left( \Lambda(X, v) - \Lambda(\hat{X},v_m) \right) \dt \notag\\
 & = E_\tR \Ad_{\hat{X}_\tR} \left( \Lambda(X, v) \dt - \Lambda(\hat{X},v)\dt - \Upsilon [Q^{\frac{1}{2}} \td w ]   \right) \notag\\
&  \begin{multlined}[b]=E_\tR \Ad_{\hat{X}_\tR} \left( \Lambda(X, v)  - \Lambda(\hat{X},v)  \right) \dt \\
 \qquad\qquad - \tD \tL_{E_\tR} \Ad_{\hat{X}_\tR} \Upsilon [Q^{\frac{1}{2}} \td w ].  \end{multlined} 
\label{eq:right_handed_error_linearisation}
\end{align}\normalsize
Note that in both cases the noise enters linearly in the $\Upsilon$ term. 
In the case of the left-handed error coordinates there is a simple $\tD \tR_{E_\tL}$ that is suppressed in the linearisation \eqref{eq:BL}. 
Conversely, in the right-handed error term there is an $\Ad_{\hat{X}_\tR} \Upsilon$ term through which the noise enters \eqref{eq:BR}. 
While the noise appears to be "transformed" by the adjoint, it is still linear in the underlying noise process as $\Ad_{\hat{X}_\tR}$ is a linear operator on the Lie algebra. 

\subsubsection{Predict}
Consider the system dynamics \eqref{eq:dynamicsCT}, the L-IEKF and R-IEKF dynamics in the predict step are obtained through propagating the reference state by the continuous-time system model, and propagating the covariance by linearising the error dynamics \eqref{eq:left_handed_error_linearisation} and \eqref{eq:right_handed_error_linearisation} of $E_\tL$ and $E_\tR$, respectively.

\begin{definition}[L-IEKF prediction dynamics]\label{def:left_prediction}
    Let   $(\mu_\tL,\allowbreak \hat{X}_\tL, \allowbreak\Sigma_\tL)$ denote the L-IEKF state with $\mu_\tL = 0$ at time $t_k$. 
    Note that the subscript $k$ has been dropped.
    The prediction is given by integrating the dynamics \eqref{eq:dynamicsCT} over the interval $[t_k, t_{k+1}]$:
    \begin{align*}
        \dot{\mu}_\tL &= 0, & \mu_\tL(t_k) = 0 \\
        \dot{\hat{X}}_\tL &= f_{v_m}(\hat{X}_\tL), & \hat{X}_\tL(t_k) = \hat{X}_\tL\\
        \dot{\Sigma}_\tL &= A_\tL \Sigma_\tL + \Sigma_\tL A_\tL^\top + B_\tL Q  B_\tL^\top & \Sigma_\tL(t_k) = \Sigma_\tL
    \end{align*}
    where $A_{\tL}$ and $B_{\tL}$ are the linearisation of the system \eqref{eq:dynamicsCT}
    \begin{align}
        A_\tL &:= \tD_X|_{\hat{X}_\tL} \Lambda(X, v_m) \cdot \tD \tL_{\hat{X}_\tL}(I) - \ad^\vee_{\Lambda(\hat{X}_\tL, v_m)}, \label{eq:AL} \\
        B_\tL &:= \tD_{u}|_{v_m} \Lambda(\hat{X}_\tL, u) = \Upsilon.\label{eq:BL}
    \end{align}
The output is 
    \begin{align*}
        \mu_\tL^- &= 0, \\
        \hat{X}_\tL^- &= \hat{X}_\tL(t_{k+1}), \\
        \Sigma_\tL^- &= \Sigma_\tL(t_{k+1} ).
        \end{align*}
\end{definition}

\begin{definition}[R-IEKF prediction dynamics]\label{def:right_prediction}
    Let $(\mu_\tR,\allowbreak \hat{X}_\tR, \Sigma_\tR)$ denote the R-IEKF state with $\mu_\tR = 0$ at time $t_k$.  Note that the subscript $k$ has been dropped. 
    The prediction is given by integrating the dynamics \eqref{eq:dynamicsCT} over the interval $[t_k, t_{k+1}]$:
    \begin{align*}
        \dot{\mu}_\tR &= 0, & \mu_\tR(t_k) = 0\\
        \dot{\hat{X}}_\tR &= f_{v_m}(\hat{X}_\tR),  & \hat{X}_\tR(t_k) = \hat{X}_\tR\\
        \dot{\Sigma}_\tR &= A_\tR \Sigma_\tR + \Sigma_\tR A_\tR^\top + B_\tR Q B_\tR^\top  & \Sigma_\tL(t_k) = \Sigma_\tL 
    \end{align*}
    where $A_\tR, B_\tR$ are the linearisation matrices, defined by
    \begin{align}
        A_\tR &= \Ad_{\hat{X}_\tR} \tD_X|_{\hat{X}_\tR} \Lambda(X, v_m)\cdot \tD \tR_{\hat{X}_\tR}(I), \label{eq:AR}\\
        B_\tR &= \Ad_{\hat{X}_\tR} \tD_{u}|_{v_m} \Lambda(\hat{X}_\tR, u) 
        = \Ad_{\hat{X}_\tR}^\vee \Upsilon \label{eq:BR}
    \end{align}
Note the $\Ad_{\hat{X}_\tR}$ term that naturally appears in the $B_\tR$ matrix associated with the change of handedness in the representation.
In particular, the $B_\tR$ term is not constant along the trajectory. 
The output is 
    \begin{align*}
        \mu_\tR^- &= 0, \\
        \hat{X}_\tR^- &= \hat{X}_\tR(t_{k+1}), \\
        \Sigma_\tR^- &= \Sigma_\tR(t_{k+1}).
        \end{align*}
\end{definition}

With the initial condition represented by $\GP_{\hat{X}(t_k)}(0,\allowbreak \Sigma(t_k))$ at time $t_k$, the prediction step consists of integrating the ODEs given in Definitions \ref{def:left_prediction} and \ref{def:right_prediction} until $t_{k+1}$ when the new measurement comes in.
At the end of the propagation step, the estimated state is represented by $\GP_{\hat{X}^-(t_{k+1})}(0,\allowbreak \Sigma^-(t_{k+1}))$. \\

\subsubsection{General Measurement Update}\label{sub:general_update}
The aim of this step is to update the filter estimate by incorporating the new measurement $y_{k+1}$.
We first consider the measurement update for the generic measurement model \eqref{eq:measurement_generic}.
Note that the subscript step $k+1$ has been suppressed in all states and replaced by the filter handedness in all equations. 
We also suppress the index $y = y_{k+1}$ on the measurement to respect the convention, although the same measurement is used for both handedness. 

\begin{definition}[L-IEKF measurement update]
    Let $(\mu_\tL^-,\allowbreak \hat{X}_\tL^-,\Sigma_\tL^-)$ denote the L-IEKF state with $\mu_\tL^-=0$ and let $y$ be a measurement as described in \eqref{eq:measurement_generic}.
    The posterior L-IEKF state $(\mu_\tL^+, \hat{X}_\tL^-, \Sigma_\tL^+)$ is obtained through the \emph{update step} by
    \begin{align*}
        S_\tL &:= C_\tL \Sigma_\tL^- C_\tL^\top + R, \\
        K_\tL &:= \Sigma_\tL^- C_\tL^\top S_\tL^{-1}, \\
        \mu_\tL^+ &= K_\tL (y - h(\hat{X}_\tL^-)), \\
        \Sigma_\tL^+ &= (I - K_\tL C_\tL) \Sigma_\tL^-,
    \end{align*}
    where $C_\tL$ is the linearisation matrix defined by
    \begin{align}\label{eq:CL}
        C_\tL &:= \tD_X|_{\hat{X}_\tL^-} h(X)\cdot \tD \tL_{\hat{X}_\tL^-}(I).
    \end{align}
\end{definition}

\begin{definition}[R-IEKF measurement update]
    Let $(\mu_\tR^-,\allowbreak \hat{X}_\tR^-, \Sigma_\tR^-)$ denote the R-IEKF state with $\mu_\tR^-=0$ and let $y$ be a measurement as described in \eqref{eq:measurement_generic}.
    The posterior R-IEKF state $(\mu_\tR^+, \hat{X}_\tR^-, \Sigma_\tR^+)$ is obtained through the \emph{update step} by
    \begin{align*}
        S_\tR &:= C_\tR \Sigma_\tR^- C_\tR^\top + R, \\
        K_\tR &:= \Sigma_\tR^- C_\tR^\top S_\tR^{-1}, \\
        \mu_\tR^+ &= K_\tR (y - h(\hat{X}_\tR^-)), \\
        \Sigma_\tR^+ &= (I - K_\tR C_\tR) \Sigma_\tR^-,
    \end{align*}
    where $C_\tR$ is the linearisation matrix defined by
    \begin{align}\label{eq:CR}
        C_\tR &:= \tD_X|_{\hat{X}_\tR^-} h(X)\cdot \tD \tR_{\hat{X}_\tR^-}(I).
    \end{align}
\end{definition}

The two update steps are identical apart from notation and the $\tD \tL_{\hat{X}_\tL}$ or $\tD \tR_{\hat{X}_\tR}$ terms in the $C$ matrix. 

In the update step, the measurement $y$ and noise process $\nu \sim \GP(0, R)$ are used to estimate a fused density $\GP_{\hat{X}^-(t_{k+1})}(\mu^+(t_{k+1}), \Sigma^+(t_{k+1}))$. 
Note that the update step involves a Bayesian fusion in local exponential coordinates, which leads to a non-zero offset $\mu^+(t_{k+1})$ and does not change the reference $\hat{X}^-(t_{k+1})$. 

One could define $X_\tL^+ = X_\tL^-$ (and $X_\tR^+ = X_\tR^-$) in order to make the state notation consistent. 
However, it is better to use the `$-$' notation for the reference in order to reinforce that the reference exponential coordinates do not change in the update step. \\

\subsubsection{Reset}\label{sub:Reset}
The reset step transforms the posterior $\GP_{\hat{X}^-(t_{k+1})}  (\mu^+\allowbreak(t_{k+1}), \Sigma^+(t_{k+1}))$ into a new distribution with zero offset $\GP_{\hat{X}(t_{k+1})}(0, \Sigma(t_{k+1}))$ at the tangent coordinates at the new reference state.
Note that in much of the EKF literature the reset step is simply a reset of the reference 
\[
\hat{X} (t_{k+1}) = \hat{X}^- (t_{k+1}) \expG(\mu^+(t_{k+1}))
\]
without any corresponding consideration of the consequences of the change of coordinates on the covariance $\Sigma^+(t_{k+1})$. 
Theoretical arguments from both geometric \cite{geGeometricPerspectiveFusing2024} and analytic \cite{muellerCovarianceCorrectionStep2017} perspectives have conclusively shown that the change in the reference state induces a change in the covariance $\Sigma$ when the state space is not globally Euclidean. 
Conceptually, changing the reference point of a concentrated Gaussian changes the coordinates in which the pdf is expressed and should change the algebraic form of the covariance. 
The authors claim that failing to model the reset correctly is a major source that leads to differences between the L-IEKF and R-IEKF in the literature, which is supported by both analytical derivations and experimental results presented in this paper.

\begin{remark}
The reset step was initially introduced as part of the MEKF framework \cite{markleyAttitudeErrorRepresentations2003} which resets the reference point of the state estimate.
Although the original work claimed that the reset step should not affect the covariance as it does not add any new information, there have been subsequent studies on attitude estimation \cite{zanettiQuaternionEstimationNorm2006, reynoldsAsymptoticallyOptimalAttitude2008} that suggest otherwise. 
This was later confirmed by \cite{muellerCovarianceCorrectionStep2017, gillFullOrderSolutionAttitude2020, markleyErrorCovarianceResetMultiplicative2023} which demonstrated that the reset changes the reference frame for the attitude error covariance and show the reset step improves the overall filtering performance.
In a separate but closely related stream of work \cite{bourmaudDiscreteExtendedKalman2013,bourmaudContinuousDiscreteExtendedKalman2015}, the authors presented a general Lie-group EKF framework that incorporates the reset step to account for the change of coordinates in which the covariance is expressed in the EKF update. 
In \cite{mahonyObserverDesignNonlinear2022,geEquivariantFilterDesign2022}, the reset step was introduced to the equivariant filter framework as a curvature correction term based on the parallel transport operator.
Later in \cite{geGeometricPerspectiveFusing2024}, the authors proposed using the differential of the exponential map for the reset step in the context of fusing Gaussian distributions on Lie groups.
A similar formula can be derived by treating the EKF update as a maximum likelihood optimisation and computing the new covariance as the Hessian of the log-likelihood \cite{bonnabelIntrinsicCramerRaoBound2015a, geGeometryExtendedKalman2025}. 
\end{remark}

The reset step is defined by
\begin{definition}[L-IEKF reset]
    Let $(\mu_\tL^+, \hat{X}_\tL^-, \Sigma_\tL^+)$ denote the L-IEKF state.
    Then the \emph{reset step} is defined by
    \begin{align*}
        \mu_\tL &= 0, \\
        \hat{X}_\tL &= \hat{X}_\tL^- \expG(\mu_\tL^+), \\
        \Sigma_\tL &= J_\tL \Sigma_\tL^+ J_\tL^\top,
    \end{align*}
    where $J_\tL$ is the Jacobian matrix defined by
    \begin{align}
        J_\tL &:= \tD \tL_{\expG(\mu_\tL^+)}^{-1}\cdot \tD_\Delta|_{\mu_\tL^+}\expG(\Delta). \label{eq:JL}
    \end{align}
\end{definition}

\begin{definition}[R-IEKF reset]
    Let $(\mu_\tR^+, \hat{X}_\tR^-, \Sigma_\tR^+)$ denote the R-IEKF state.
    Then the \emph{reset step} is defined by
    \begin{align*}
        \mu_\tR &= 0, \\
        \hat{X}_\tR &= \expG(\mu_\tR^+) \hat{X}_\tR^-, \\
        \Sigma_\tR &= J_\tR \Sigma_\tR^+ J_\tR^\top,
    \end{align*}
    where $J_\tR$ is the Jacobian matrix defined by
    \begin{align}
        J_\tR &:= \tD \tR_{\expG(\mu_\tR^+)}^{-1}\cdot \tD_\Delta|_{\mu_\tR^+}\expG(\Delta). \label{eq:JR}
    \end{align}
\end{definition}

\begin{remark}
    The Jacobian matrices $J_\tL$ and $J_\tR$ are commonly referred to as the \emph{right} and \emph{left} Jacobians in most literature, respectively \cite{chirikjianStochasticModelsInformation2011}.
    We use the notation $J_\tL$ and $J_\tR$ to align with the handedness of the filter, as well as the fact that these Jacobians are derived through the left and right trivialisation of the group, respectively.
\end{remark}

The reset step transforms the posterior, which is a concentrated Gaussian with non-zero offset, into a zero-offset distribution by choosing a new reference state.
The distribution after reset is represented with $\GP_{\hat{X}(t_{k+1})}(0, \Sigma(t_{k+1}))$ which concludes a complete filter iteration. 

\begin{remark}
Note that the reset step was not included in the original IEKF framework \cite{barrauInvariantExtendedKalman2017}. 
The original IEKF formulation without reset admits a number of theoretical properties including stability \cite{barrauInvariantExtendedKalman2017}, preservation of equality constraints  \cite{barrauExtendedKalmanFiltering2020}, and consistency in the presence of unobservability due to invariance \cite{brossardExploitingSymmetriesDesign2019}. 
The IEKF with reset inherits some of these properties. 
The reset step preserves the filter's local asymptotic stability since the effect of the Jacobian can be locally bounded and absorbed into the proof arguments in \cite{barrauInvariantExtendedKalman2017}. 
Respect of precise information is also preserved since the equality constraint imposes a closure condition on the Lie subalgebra and this preserves the covariance update to the equality constraint tangent space. 
The authors do not believe that IEKF with reset preserves consistency in the presence of unobservability due to invariance. 
In this case, the Jacobian mixes information from observable directions into the unobservable space and results in growth of Fisher information in the unobservable directions. 
Furthermore, we undertake an ablation study in \S\ref{sec:experiment} that shows that the IEKF without reset appears to have better transient response in the INS filtering problem, a property that may be associated with the reset breaking the global structure of Bayesian fusion of Gaussians on the linear position and velocity spaces. 
In conclusion, there are arguments for considering the IEKF with reset and arguments for considering the IEKF without reset for different scenarios and applications. 
\end{remark}

\subsection{Left-invariant measurement model and update}
\label{sub:invariant_measurement}
In Sec.~\ref{sub:general_update}, we considered the general measurement model \eqref{eq:measurement_generic} and the corresponding update step.
However, there are certain types of measurements that draw attention due to their nice linearisation properties, namely \emph{left-invariant} and \emph{right-invariant} measurements \cite{barrauInvariantExtendedKalman2017}.
The common perception is that such measurements can lead to a state-independent linearisation matrix when being used with the IEKF with the corresponding handedness, which improves the filter performance.

In this section, we will illustrate the left-invariant measurement model and the corresponding update step for the L-IEKF and R-IEKF.
The right-invariant measurement and the corresponding update step follow similarly.

\begin{definition}[Left-invariant measurement]\label{def:left_invariant_measurement}
A measurement $y \in \R^n$ is termed left-invariant if, for some $\mr{y}\in \R^n$, it satisfies the following property:
\begin{align*}
    y = h(X) + \eta = \rho(X, \mr{y}) + \eta,\qquad\eta \sim \GP(0, R)
\end{align*}
where $\rho : \grpG \times \R^n \to \R^n$ is a left group action, that is, 
\begin{align*}
    \rho(X_1 X_2, y) &= \rho(X_1, \rho(X_2, y)), &
    \rho(I, y) &= y,
\end{align*}
for all $X_1, X_2 \in \grpG$ and all $y \in \R^n$.
\end{definition}

Given a left-invariant measurement $y = \rho(X, \mr{y}) + \eta$, let $X' \in \grpG$ and observe that 
\[
\rho(X',y) = \rho(X',\rho(X, \mr{y}) + \eta) = \rho(X' X, \mr{y}) + \rho(X',\eta)
\]
where we exploit the linearity of the group action to factor out the noise process. 
This is a form of equivariance of the output \cite{vangoorEquivariantFilterEqF2023}. 
Note, however, that this operation transforms the noise.

In most literature \cite{barrauInvariantExtendedKalman2017}, the perception is that the L-IEKF should be preferred when the measurement model is left-invariant, and vice versa for the R-IEKF.
Left-invariant measurements have been exploited in the L-IEKF formulation as follows.
Let $(0, \hat{X}_\tL, \Sigma_\tL)$ denote the L-IEKF state.
Then for a measurement $y = h(X) + \eta$, one may define the pseudo measurement $d = \rho(\hat{X}_\tL^{-1}, y)$ and find that
\begin{align*}
    d &= \rho(\hat{X}_\tL^{-1}, y) = \rho(\hat{X}_\tL^{-1}, h(X)+\eta) \\
    &= \rho(\hat{X}_\tL^{-1}, \rho(X, \mr{y})+\eta) \\
    &= \rho(\hat{X}_\tL^{-1}, \rho( \hat{X}_\tL E_\tL, \mr{y})+\eta).
\end{align*}
Since the action $\rho$ is a linear group action, this simplifies to
\begin{align*}
    d
    &= \rho(\hat{X}_\tL^{-1}, \rho(\hat{X}_\tL E_\tL, \mr{y})) + \rho(\hat{X}_\tL^{-1}, \eta) \notag \\
    &= \rho(E_\tL, \mr{y}) + \rho(\hat{X}_\tL^{-1}, \eta) = h(E_\tL) + \rho(\hat{X}_\tL^{-1}, \eta).
\end{align*}
The linearisation about $E_\tL = I$ and $\eta = 0$ yields
\begin{align}\label{eq:linearisation_rightmeasurement_L}
    d = \mr{y} + \tD_X|_I h(X) \varepsilon_\tL + \tD_z|_{h(\hat{X}_\tL)} \rho(\hat{X}_\tL^{-1}, z) \eta
    + O(\vert \eta, \varepsilon_\tL \vert^2).
\end{align}
where $E_\tL = I + \varepsilon_\tL + O (\vert\varepsilon_\tL \vert^2)$. 
This leads to the L-IEKF measurement update with left-invariant measurements.
\begin{definition}[L-IEKF left-invariant measurement update]\label{def:left_measurement_update}
    Let $(\mu_\tL^-, \hat{X}_\tL^-, \Sigma_\tL^-)$ denote the L-IEKF state with $\mu_\tL^-=0$, and let $y$ be a left-invariant measurement as described in \ref{def:left_invariant_measurement}.
    The posterior L-IEKF state $(\mu_\tL^+, \hat{X}_\tL^-, \Sigma_\tL^+)$ is obtained through the \emph{left update step}
    \begin{align*}
        S_\tL &:= C_\tL \Sigma_\tL^- C_\tL^\top + D_\tL R D_\tL^\top, \\
        K_\tL &:= \Sigma_\tL^- C_\tL^\top S_\tL^{-1}, \\
        \mu_\tL^+ &= K_\tL (\rho({\hat{X}_\tL^-}^{-1}, y) - \mr{y}), \\
        \Sigma_\tL^+ &= (I - K_\tL C_\tL) \Sigma_\tL^-,
    \end{align*}
    where $C_\tL, D_\tL$ are the linearisation matrices defined by
    \begin{align}
        C_\tL &:= \tD_X|_I h(X), \label{eq:CL_left_y_L-IEKF}\\
        D_\tL &:=  \tD_z|_{h(\hat{X}_\tL^-)} \rho_{{\hat{X}_\tL^-}^{-1}}(z).\label{eq:DL_left_y_L-IEKF}
    \end{align}
\end{definition}

As shown in Def.~\ref{def:left_measurement_update}, by using the pseudo measurement $d$, the L-IEKF now admits a state-independent linearisation matrix $C_\tL = \tD_X|_I h(X)$, which is the main reason why L-IEKF is always preferred over R-IEKF in this case.
However, note that the same pseudo-measurement construction is also possible for R-IEKF.
Consider an R-IEKF and the same pseudo measurement $d = \rho(\hat{X}_\tR^{-1}, y)$, one now has 
\begin{align*}
    d &= \rho(\hat{X}_\tR^{-1}, y) = \rho(\hat{X}_\tR^{-1}, \rho(X, \mr{y})+\eta) \\ 
    &= \rho(\hat{X}_\tR^{-1}, \rho( X \hat{X}_\tR^{-1} \hat{X}_\tR, \mr{y})+\eta) \\ 
    &= \rho(\hat{X}_\tR^{-1}, \rho( E_\tR \hat{X}_\tR, \mr{y})+\eta).
\end{align*}
Linearising about $E_\tR = I$ and $\eta = 0$ yields
\begin{align}
    d = \mr{y} + \tD_X|_I h(X) \Ad_{\hat{X}_\tR^{-1}} \varepsilon_\tR + &\tD_z|_{h(\hat{X}_\tR)} \rho(\hat{X}_\tR^{-1}, z) \eta \notag\\
    &\qquad\qquad + O(\vert \eta, \varepsilon_\tR \vert^2).\label{eq:linearisation_rightmeasurement_R}
\end{align}
This leads to the left-invariant measurement updates for the R-IEKF using pseudo-measurements.

\begin{definition}[R-IEKF left-invariant measurement update]
    Let $(\mu_\tR^-, \hat{X}_\tR^-, \Sigma_\tR^-)$ denote the R-IEKF state with $\mu_\tR^-=0$, and let $y$ be a left measurement as described in \ref{def:left_invariant_measurement}.
    The posterior R-IEKF state $(\mu_\tR^+, \hat{X}_\tR^-, \Sigma_\tR^+)$ is obtained through the \emph{right update step}
    \begin{align*}
        S_\tR &:= C_\tR \Sigma_\tR^- C_\tR^\top + D_\tR R D_\tR^\top, \\
        K_\tR &:= \Sigma_\tR^- C_\tR^\top S_\tR^{-1}, \\
        \mu_\tR^+ &= K_\tR (\rho({\hat{X}_\tR^-}^{-1}, y) - \mr{y}), \\
        \Sigma_\tR^+ &= (I - K_\tR C_\tR) \Sigma_\tR^-,
    \end{align*}
    where $C_\tR, D_\tR$ are the linearisation matrices defined by
    \begin{align}
        C_\tR &:= \tD_X|_I h(X) \cdot \Ad_{{\hat{X}_\tR^-}^{-1}}^\vee, \label{eq:CL_left_y_R-IEKF}\\
        D_\tR &:=  \tD_z|_{h(\hat{X}_\tR^-)} \rho_{{\hat{X}_\tR^-}^{-1}}( z).\label{eq:DL_left_y_R-IEKF}
    \end{align}
\end{definition}

\begin{remark}
In much of the literature, it is always encouraged to use L-IEKF over the R-IEKF when a left-invariant measurement is available, because of the state-independent linearisation matrix $C_\tL = \tD_X|_I h(X)$ in \eqref{eq:linearisation_rightmeasurement_L}.
However, as shown in \eqref{eq:linearisation_rightmeasurement_L}, when using the pseudo measurement $d$ instead of $y$, the new noise process now depends on the linearisation matrix $\tD_z|_{h(\hat{X}_\tL^-)} \rho(\hat{X}^{-1}, z)$, which is state-dependent.
This finding is also observed in \cite{liErrorDynamicsAffine2025}.
\end{remark}

\section{Equivalence of Left and Right IEKF with reset}
\label{sec:equivalence}
In this section, we present our main claim that the L-IEKF and R-IEKF with reset always yield equivalent estimates of the state at each filter step.
In Section~\ref{sub:definition}, we define the equivalence of the L-IEKF and R-IEKF with reset based on the underlying distribution of the state estimate represented by concentrated Gaussians.
The main equivalence result is presented in Theorem~\ref{thm:IEKF_equivalence} in the following section.
The proof of the theorem relies on proving equivalence through each filter step, which is shown in Lemma~\ref{lem:prediction_equivalence}-\ref{lem:reset_equivalence}.

\subsection{Definition of equivalence}
\label{sub:definition}
The equivalence of the L-IEKF and R-IEKF with reset is defined based on the underlying distribution of the state estimate represented by the concentrated Gaussians.

\begin{lemma}\label{lemma:cgd_equiv}
    Consider a L-CGD $\GP^\tL_{\hat{X}_\tL}(\mu_\tL, \Sigma_\tL)$ and a R-CGD $\GP^\tR_{\hat{X}_\tR}(\mu_\tR, \Sigma_\tR)$ defined by \eqref{eq:L_CGD} and \eqref{eq:R_CGD} respectively. 
    The two distributions are equivalent (as functions on $\grpG$) if 
    \begin{align}
        \mu_\tR &= \Ad_{\hat{X}_\tL} \mu_\tL, &
        \hat{X}_\tR &= \hat{X}_\tL, &
        \Sigma_\tR &= \Ad_{\hat{X}_\tL}^\vee \Sigma_\tL {\Ad_{\hat{X}_\tL}^\vee}^\top.
\label{eq:left_right_correspondence_condition}
    \end{align}
\end{lemma}

\begin{proof}
    We prove this lemma by showing that the log likelihood of left and right distributions is the same under the conditions of the lemma, that is,
    \begin{align*}
        \Lyap_\tL(\mu_\tL, \hat{X}_\tL, \Sigma_\tL)
        &= \Lyap_\tR(\mu_\tR, \hat{X}_\tR, \Sigma_\tR).
    \end{align*}
    This is shown by straightforward computation.
    \begin{align*}
        \Lyap_\tR&(\mu_\tR, \hat{X}_\tR, \Sigma_\tR)\\
        &= \frac{1}{2} \vert \log(X \hat{X}_\tR^{-1}) - \mu_\tR \vert_{\Sigma_\tR^{-1}}^2 \\
        &= \frac{1}{2} \vert \log(X \hat{X}_\tL^{-1}) - \Ad_{\hat{X}_\tL} \mu_\tL \vert_{(\Ad_{\hat{X}_\tL}^\vee \Sigma_\tL {\Ad_{\hat{X}_\tL}^\vee}^\top)^{-1}}^2 \\
        &= \frac{1}{2} \vert  \Ad_{\hat{X}_\tL} \log(\hat{X}_\tL^{-1} X) - \Ad_{\hat{X}_\tL} \mu_\tL \vert_{(\Ad_{\hat{X}_\tL}^\vee \Sigma_\tL {\Ad_{\hat{X}_\tL}^\vee}^\top)^{-1}}^2  \\
        &= \frac{1}{2} \vert  \log(\hat{X}_\tL^{-1} X) -  \mu_\tL \vert_{\Sigma_\tL^{-1}}^2 \\
        &= \Lyap_\tL(\mu_\tL, \hat{X}_\tL, \Sigma_\tL),
    \end{align*}
    for any $X_\tL = X_\tR \in \grpG$.
\end{proof}

Concretely, this lemma shows that the L-IEKF and R-IEKF, which use left and right concentrated Gaussians for state representation respectively, have the exact same log likelihood function when the states correspond according to the condition \eqref{eq:left_right_correspondence_condition}.
This leads to the following definition of equivalence between the L-IEKF and R-IEKF with reset.
\begin{definition}\label{def:equivalent}
Let $(\mu_\tL, \hat{X}_\tL, \Sigma_\tL)$ and $(\mu_\tR, \hat{X}_\tR, \Sigma_\tR)$ denote the states of the L-IEKF and R-IEKF, respectively.
We say the L-IEKF and R-IEKF states are \emph{equivalent} if \eqref{eq:left_right_correspondence_condition} holds. 
\end{definition}

\subsection{Equivalence of filter steps}
\label{sub:equivalence}
In this section, we first present the main theorem, which shows that the evolution of the left and right filter estimates matches, and hence the likelihood functions are equivalent.
Its proof relies on the lemmas that are proven later in the section. 
Note that we will re-introduce the time index $t_k, t_{k+1}$ when necessary to clarify the time evolution of the filter propagation.

\begin{theorem}\label{thm:IEKF_equivalence}
Let $(\mu_\tL, \hat{X}_\tL, \Sigma_\tL)$ and $(\mu_\tR, \hat{X}_\tR, \Sigma_\tR)$ denote the states of the L-IEKF and R-IEKF with reset, respectively.
Assume that the measurements $y_k$ admit left-invariance (Section~\ref{sub:invariant_measurement}). 
If the initial conditions of the states are equivalent (at a time $t_0$), 
\begin{align}
\mu_\tR(t_0) & = 0 = \Ad_{X_\tL(t_0)} \mu_\tL(t_0) \notag \\ 
\hat{X}_\tR(t_0)  & = \hat{X}_\tL(t_0) \notag \\
\Sigma_\tR (t_0) & = \Ad_{X_\tL(t_0)}^\vee\Sigma_\tL(t_0) {\Ad_{X_\tL(t_0)}^\vee}^\top, \notag 
\end{align}
with $\mu_\tL(t_0) = \mu_\tR(t_0) = 0$, then the states are equivalent for all time $t \geq t_0$.
\end{theorem}

\begin{proof}
Lemma \ref{lem:prediction_equivalence} proves that if the initial state of the L-IEKF and R-IEKF
\[
(\mu_\tL, \hat{X}_\tL, \Sigma_\tL) = (\mu_\tL(t_k), \hat{X}_\tL(t_k), \Sigma_\tL(t_k)) \] and \[ 
(\mu_\tR, \hat{X}_\tR, \Sigma_\tR) = (\mu_\tR(t_k), \hat{X}_\tR(t_k), \Sigma_\tR(t_k)) 
\]
are equivalent, then the filter state after propagation
$(\mu_\tL^-,\allowbreak \hat{X}_\tL^-, \Sigma_\tL^-)$ and $(\mu_\tR^-, \hat{X}_\tR^-, \Sigma_\tR^-)$
are equivalent.

Lemma \ref{lem:measurement_equivalence_right} proves that if 
$(\mu_\tL^-, \hat{X}_\tL^-,\allowbreak \Sigma_\tL^-)$ and $(\mu_\tR^-, \hat{X}_\tR^-,\allowbreak \Sigma_\tR^-)$  are equivalent then the filter states after the measurement update
$(\mu_\tL^+, \hat{X}_\tL^-, \Sigma_\tL^+)$ and $(\mu_\tR^+, \hat{X}_\tR^-, \Sigma_\tR^+)$  are equivalent. 

Lemma \ref{lem:reset_equivalence} proves that if $(\mu_\tL^+, \hat{X}_\tL^-, \Sigma_\tL^+)$ and $(\mu_\tR^+, \hat{X}_\tR^-,\allowbreak \Sigma_\tR^+)$  are equivalent 
then the filter states after the reset step
\[
(\mu_\tL(t_{k+1}), \hat{X}_\tL(t_{k+1}), \Sigma_\tL(t_{k+1})) = (\mu_\tL, \hat{X}_\tL, \Sigma_\tL)\]
and
\[(\mu_\tR(t_{k+1}), \hat{X}_\tR(t_{k+1}), \Sigma_\tR(t_{k+1}))  = (\mu_\tR, \hat{X}_\tR, \Sigma_\tR)
\]
are equivalent. 

By induction, this completes the proof.
\end{proof}
\begin{remark}
    We only present the technical results with left-invariant measurements since this matches the setup in the INS problem. 
    However, the equivalence holds for any type of measurements, and it is straightforward to extend the proof to the case of right-invariant measurements or general measurements.
    We omit the detailed proofs for other types of measurements.
    A similar analysis has recently been performed in \cite{maurerEquivalenceLeftRightInvariant2025} for the case of general measurement update.
\end{remark}

\begin{lemma}[Predict step equivalence]\label{lem:prediction_equivalence}
    Let $(\mu_\tL, \hat{X}_\tR, \Sigma_\tL)$ and $(\mu_\tR, \hat{X}_\tR, \Sigma_\tR)$ be the states of the L-IEKF and R-IEKF, respectively, with $\mu_\tL = \mu_\tR = 0$.
    Suppose that the L-IEKF and R-IEKF states are equivalent (Def.~\ref{def:equivalent}) at a time $t_0$.
    Then the L-IEKF and R-IEKF states remain equivalent for as long as the prediction dynamics are applied.
\end{lemma}
    
\begin{proof}
    This is shown by applying the uniqueness of smooth ODE solutions. 
    Let $P := \Ad_{\hat{X}_\tL} \Sigma_\tL \Ad_{\hat{X}_\tL}^\top$ 
    where $(0,X_\tL,\Sigma_\tL)$ is the solution of the left-invariant prediction dynamics. 
    Then $P(t_0) = \Sigma_\tR(t_0)$ and the dynamics of $P$ are given by 
    \begin{align*}
        \dot{P}
        &= \ddt \Ad_{\hat{X}_\tL}^\vee \Sigma_\tL {\Ad_{\hat{X}_\tL}^\vee}^\top \\
        &= \Ad_{\hat{X}_\tL}^\vee (A_\tL \Sigma_\tL + \Sigma_\tL A_\tL^\top + B_\tL Q B_\tL^\top) {\Ad_{\hat{X}_\tL}^\vee}^\top
        \\&\hspace{1cm}
        + \Ad_{\hat{X}_\tL}^\vee \ad_{\Lambda(\hat{X}_\tL, v_m)}^\vee \Sigma_\tL {\Ad_{\hat{X}_\tL}^\vee}^\top
        \\&\hspace{1cm}
        + \Ad_{\hat{X}_\tL}^\vee \Sigma_\tL {\ad_{\Lambda(\hat{X}_\tL, v_m)}^\vee}^\top {\Ad_{\hat{X}_\tL}^\vee}^\top \\
        &= (\Ad_{\hat{X}_\tL}^\vee A_\tL {\Ad_{\hat{X}_\tL}^\vee}^{-1} + \Ad_{\hat{X}_\tL}^\vee \ad_{\Lambda(\hat{X}_\tL, v_m)}^\vee {\Ad_{\hat{X}_\tL}^\vee}^{-1}) P 
        \\&\hspace{0.2cm}
        + P (\Ad_{\hat{X}_\tL}^\vee A_\tL {\Ad_{\hat{X}_\tL}^\vee}^{-1} + \Ad_{\hat{X}_\tL}^\vee \ad_{\Lambda(\hat{X}_\tL, v_m)}^\vee {\Ad_{\hat{X}_\tL}^\vee}^{-1})^\top
        \\&\hspace{0.2cm}
        + \Ad_{\hat{X}_\tL}^\vee B_\tL Q B_\tL^\top  {\Ad_{\hat{X}_\tL}^\vee}^\top.
    \end{align*}
    where $A_\tL$ and $B_\tL$ are the linearisation matrices given in Def.~\ref{def:left_prediction}. 
    Observe that
    \begin{align*}
        &\Ad_{\hat{X}_\tL}^\vee A_\tL {\Ad_{\hat{X}_\tL}^\vee}^{-1} + \Ad_{\hat{X}_\tL}^\vee \ad_{\Lambda(\hat{X}_\tL, v_m)}^\vee {\Ad_{\hat{X}_\tL}^\vee}^{-1} \\
        &= \Ad_{\hat{X}_\tL}^\vee (\tD_X|_{\hat{X}_\tL} \Lambda(X, v_m) \cdot\tD \tL_{\hat{X}_\tL}(I) \\
        & \qquad -\ad_{\Lambda(\hat{X}_\tL, v_m)}^\vee){\Ad_{\hat{X}_\tL}^\vee}^{-1} + \Ad_{\hat{X}_\tL}^\vee \ad_{\Lambda(\hat{X}_\tL, v_m)}^\vee {\Ad_{\hat{X}_\tL}^\vee}^{-1} \\
        &= \Ad_{\hat{X}_\tL}^\vee \tD_X|_{\hat{X}_\tL} \Lambda(X, v_m) \cdot \tD \tL_{\hat{X}_\tL}(I) {\Ad_{\hat{X}_\tL}^\vee}^{-1} \\
        &= \Ad_{\hat{X}_\tL}^\vee \tD_X|_{\hat{X}_\tL} \Lambda(X, v_m) \cdot \tD \tR_{\hat{X}_\tL}(I) \\
        &= A_\tR, 
    \end{align*}
    while 
    \[
    \Ad_{\hat{X}_\tL}^\vee B_\tL = \Ad_{\hat{X}_\tL}^\vee \Upsilon = B_\tR,
    \]
    and hence $(0,\hat{X}_\tL, P)$ obey the R-IEKF prediction dynamics.
    Uniqueness of ODE solutions implies that $\hat{X}_\tL = \hat{X}_\tR$ and $P(t) = \Sigma_\tR(t)$ for all time.
\end{proof}

\begin{remark}
    The result in Lemma~\ref{lem:prediction_equivalence} depends only on the algebraic structure of the EKF equations.  
    In particular, it does not require group-affine dynamics or any other additional properties of the system.
\end{remark}

The next result shows that the update step preserves the equivalence of the L-IEKF and R-IEKF solutions.

\begin{lemma}[Left-invariant measurement update equivalence]\label{lem:measurement_equivalence_right}
    Let $(\mu_\tL^-, \hat{X}_\tL^-, \Sigma_\tL^-)$ and $(\mu_\tR^-, \hat{X}_\tR^-, \Sigma_\tR^-)$ be the states of the L-IEKF and R-IEKF, respectively, with $\mu_\tL^-=\mu_\tR^- = 0$.
    Suppose that the states are equivalent (Def.~\ref{def:equivalent}) at the instant a left-invariant measurement $y$ is received.
    If the respective update steps are applied to both the L-IEKF and R-IEKF, then the posterior states $(\mu_\tL^+, \hat{X}_\tL^-, \Sigma_\tL^+)$ and $(\mu_\tR^+, \hat{X}_\tR^-, \Sigma_\tR^+)$ are equivalent as well.
\end{lemma}

\begin{proof}
    There is no update to the reference and $\hat{X}_\tL^- = \hat{X}_\tR^-$ remain fixed in the update step. 

    Recalling \eqref{eq:CL_left_y_L-IEKF} and \eqref{eq:CL_left_y_R-IEKF} then 
    $C_\tR = C_\tL \Ad_{{\hat{X}_\tR^-}^{-1}}^\vee$. 
    Similarly, from \eqref{eq:DL_left_y_L-IEKF} and \eqref{eq:DL_left_y_R-IEKF} then $D_\tL = D_\tR$.
    Thus one has 
    \begin{align*}
            S_\tR &= C_\tR \Sigma_\tR^- C_\tR^\top + D_\tR R D_\tR^\top\\
            & = C_\tL \Ad_{{\hat{X}_\tR^-}^{-1}}^\vee \Ad_{\hat{X}_\tL}^\vee \Sigma_\tL^- {\Ad_{\hat{X}_\tL^-}^\vee}^\top {\Ad_{{\hat{X}_\tR^-}^{-1}}^\vee}^\top C_\tL^\top \\
            & \hspace{5cm} + D_\tL R D_\tL^\top\\
        &=C_\tL \Sigma_\tL^- C_\tL^\top + D_\tL R D_\tL^\top = S_\tL.
    \end{align*}
    The updated offsets are related by
    \begin{align*}
        \mu_\tR^+
        &= K_\tR (\rho({\hat{X}_\tR^-}^{-1}, y) - \mr{y}) \\
        &= \Sigma_\tR^- C_\tR^\top S_\tR^{-1} (y - h(\hat{X}_\tR^-)) \\
        &= \Ad_{\hat{X}_\tL^-}^\vee \Sigma_\tL^- {\Ad_{\hat{X}_\tL^-}^\vee}^\top (C_\tL {\Ad_{\hat{X}_\tL^-}^\vee}^{-1})^\top S_\tL^{-1}\\
        &\hspace{5cm} \cdot(\rho({\hat{X}_\tR^-}^{-1}, y) - \mr{y}) \\
        &= \Ad_{\hat{X}_\tL^-}^\vee \Sigma_\tL^- C_\tL^\top S_\tL^{-1} (\rho({\hat{X}_\tR^-}^{-1}, y) - \mr{y}) \\
        &= \Ad_{\hat{X}_\tL^-}^\vee \mu_\tL^+.
    \end{align*}
    Finally, the updated covariances satisfy\small
    \begin{align*}
        \Sigma_\tR^+
        &= (I - K_\tR C_\tR) \Sigma_\tR^- = (I - \Sigma_\tR^- C_\tR^\top S_\tR^{-1} C_\tR) \Sigma_\tR^- \\
        &= (I - \Ad_{\hat{X}_\tL^-}^\vee \Sigma_\tL {\Ad_{\hat{X}_\tL^-}^\vee}^\top (C_\tL {\Ad_{{\hat{X}_\tL^-}}^\vee}^{-1})^\top S_\tL^{-1} C_\tL {\Ad_{\hat{X}_\tL^-}^\vee}^{-1})
        \\&\hspace{5cm} \cdot \Ad_{\hat{X}_\tL^-}^\vee \Sigma_\tL {\Ad_{\hat{X}_\tL^-}^\vee}^\top \\
        &= \Ad_{\hat{X}_\tL^-}^\vee (I - \Sigma_\tL^- {\Ad_{\hat{X}_\tL^-}^\vee}^\top (C_\tL {\Ad_{\hat{X}_\tL^-}^\vee}^{-1})^\top S_\tL^{-1} C_\tL)\Sigma_\tL^- {\Ad_{\hat{X}_\tL^-}^\vee}^\top \\
        &= \Ad_{\hat{X}_\tL^-}^\vee (I - \Sigma_\tL^- C_\tL^\top S_\tL^{-1} C_\tL)\Sigma_\tL^- {\Ad_{\hat{X}_\tL^-}^\vee}^\top \\
        &= \Ad_{\hat{X}_\tL^-}^\vee \Sigma_\tL^+ {\Ad_{\hat{X}_\tL^-}^\vee}^\top.
    \end{align*}\normalsize
    This completes the proof.
\end{proof}

The final result required is to prove equivalence under the reset step. 
\begin{lemma}[Reset step equivalence]\label{lem:reset_equivalence}
Let $(\mu_\tL^+, \hat{X}_\tL^-, \Sigma_\tL^+)$ and $(\mu_\tR^+, \hat{X}_\tR^-, \Sigma_\tR^+)$ be the states of the L-IEKF and R-IEKF, respectively, following an update step.
If the states are equivalent \eqref{eq:left_right_correspondence_condition} then after the respective reset steps are applied to both the L-IEKF and R-IEKF, the resulting states $(\mu_\tL, \hat{X}_\tL, \Sigma_\tL)$ and $(\mu_\tR, \hat{X}_\tR, \Sigma_\tR)$ are equivalent.
\end{lemma}
\begin{proof}
Consider the update of the reference states $\hat{X}_\tL^- = \hat{X}_\tL^+$ and $\hat{X}_\tR^- = \hat{X}_\tR^+$ since the update does not change the reference. 
Given \eqref{eq:left_right_correspondence_condition}, one has
\begin{align*}
    \hat{X}_\tR
    &= \expG(\mu_\tR^+) \hat{X}_\tR^-\\
    &= \expG(\Ad_{\hat{X}_\tL^-} \mu_\tL^+) \hat{X}_\tL^-
    = \hat{X}_\tL^- \expG(\mu_\tL^+)
    = \hat{X}_\tL.
\end{align*}

To prove the equivalence of the covariance matrices, we first analyse the relation between the linearisation matrices $J_\tL$ \eqref{eq:JL} and $J_\tR$ \eqref{eq:JR}
For any $\varepsilon \in \R^m$, one has\small
\begin{align*}
    J_\tR [\varepsilon]
    &= \ddz{s} \logGv(\expG(s\varepsilon^\wedge)\exp(-\mu_\tR^+)) \\
    &= \ddz{s} \logGv(\expG(s\varepsilon^\wedge)\exp(-\Ad_{\hat{X}_\tL^-} \mu_\tL^+)) \\
    &= \ddz{s} \logGv(\expG(s\varepsilon^\wedge)\hat{X}_\tL^- \exp(-\mu_\tL^+) {\hat{X}_\tL^-}^{-1}) \\
    &= \ddz{s} \logGv(\hat{X}_\tL^- \expG(\mu_\tL^+) \exp(-\mu_\tL^+) \\
    &\hspace{2cm} \expG(\Ad_{{\hat{X}_\tL^-}^{-1}}s\varepsilon^\wedge)(\hat{X}_\tL^- \expG(\mu_\tL^+))^{-1}) \\
    &= \ddz{s} \Ad_{\hat{X}_\tL^- \expG(\mu_\tL^+)}^\vee\\
    &\hspace{2cm} \logGv(\expG(-\mu_\tL^+) \expG(\Ad_{{\hat{X}_\tL^-}^{-1}}s\varepsilon^\wedge)) \\
    &= \Ad_{\hat{X}_\tL}^\vee\, J_\tL\, \Ad_{{\hat{X}_\tL^-}^{-1}}^\vee \varepsilon,
\end{align*}\normalsize
and hence $J_\tR = \Ad_{\hat{X}_\tL} J_\tL \Ad_{{\hat{X}_\tL^-}^{-1}}$.
Therefore, the reset covariance matrices are related by
\begin{align*}
    \Sigma_\tR
    &= J_\tR\, \Sigma_\tR^+\, J_\tR^\top \\
    &= \Ad_{\hat{X}_\tL}^\vee J_\tL {\Ad_{{\hat{X}_\tL^-}^{-1}}^\vee} \Ad_{\hat{X}_\tL^-}^\vee \Sigma_\tL^+ {\Ad_{\hat{X}_\tL^-}^\vee}^\top \\
    &\hspace{3cm}(\Ad_{\hat{X}_\tL}^\vee J_\tL \Ad_{{\hat{X}_\tL^-}^{-1}}^\vee)^\top \\
    &= \Ad_{\hat{X}_\tL}^\vee J_\tL \Sigma_\tL^+ J_\tL^\top {\Ad_{\hat{X}_\tL}^\vee}^\top \\
    &= \Ad_{\hat{X}_\tL}^\vee \Sigma_\tL^+ {\Ad_{\hat{X}_\tL}^\vee}^\top.
\end{align*}
This completes the proof.
\end{proof}

\begin{remark}
    Note that the reset step is critical in the whole equivalence proof, as the equivalence of the prediction step \ref{lem:prediction_equivalence} relies on the filter states being zero-mean concentrated Gaussians at the start of the prediction step.
\end{remark}

\definecolor{blue_okabe}{HTML}{0072B2}
\definecolor{green_okabe}{HTML}{009E73}
\definecolor{vermillion}{HTML}{D55E00}
\begin{figure*}[t]
    \centering
    \includegraphics[width=\linewidth]{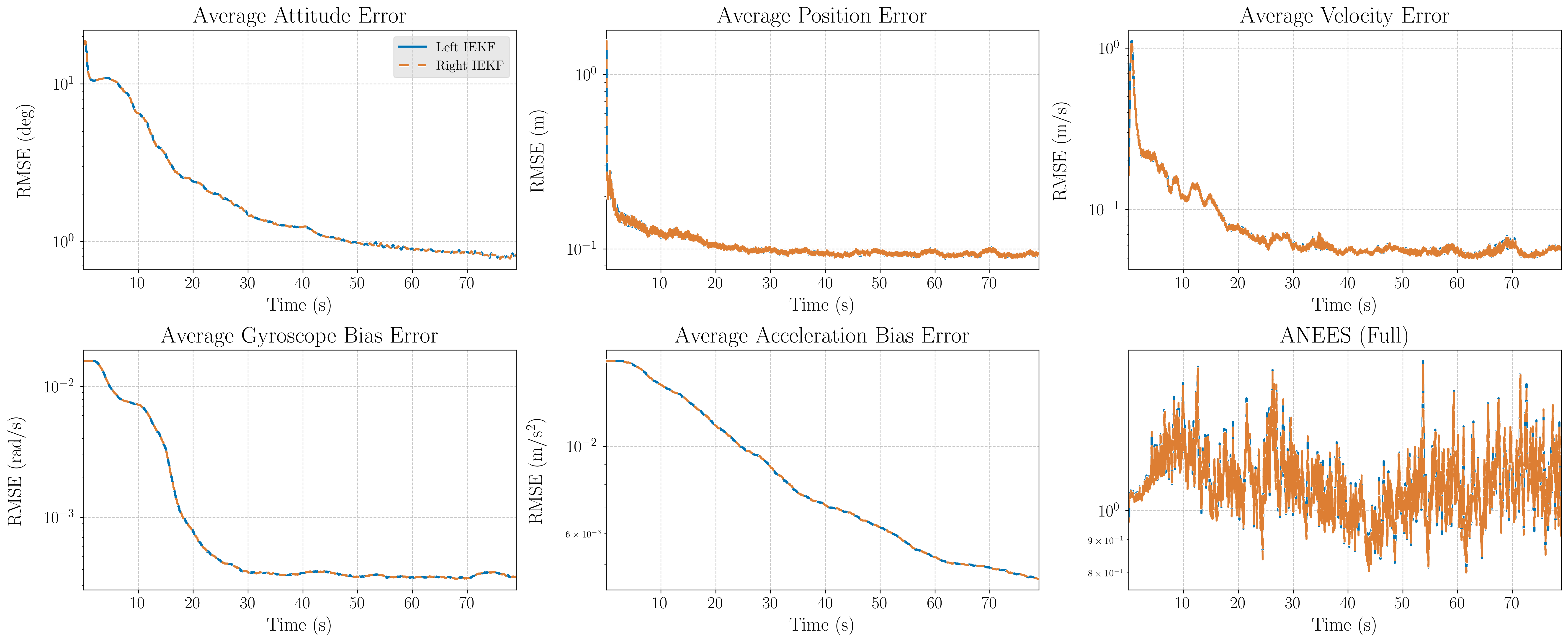}
    \caption{The RMSE of each state variable and ANEES of the L-IEKF (\textcolor{blue_okabe}{$\rule[0.5ex]{0.5cm}{1.0pt}$}) and R-IEKF (\textcolor{vermillion}{$\rule[0.5ex]{0.2cm}{1.0pt}\,\rule[0.5ex]{0.2cm}{1.0pt}$}) with reset. 
    Note that the two filters are indistinguishable in terms of performance at the current scale. This figure is used to demonstrate that the filters converge to the correct solution over time.
    }
    \label{fig:equivalence_test}
\end{figure*}

\section{Simulation Results}
\label{sec:experiment}

In this section, we test the L-IEKF and R-IEKF with reset outlined in Section~\ref{sec:IEKF_definition} on simulated data for the position-based INS problem.
We first present the primary results that demonstrate the equivalence of the L-IEKF and R-IEKF with reset in practice, and then we investigate the effect of discretisation on the equivalence.
Finally, we also present an ablation study on the effect of the reset step on the IEKF performance.

\subsection{Inertial Navigation Systems}
We will use the problem of GNSS-aided inertial navigation system (INS) as a motivating example for the paper. 
The INS problem is particularly of interest as it is a classical robotics system evolving on a Lie group, as well as a widely-studied application of the IEKF \cite{barrauGeometryNavigationProblems2023} and other symmetry-based filters \cite{fornasierEquivariantSymmetriesInertial2025}.
Consider a mobile robot equipped with an IMU that provides gyroscope and accelerometer measurements, as well as a GNSS receiver that measures the global position of the robot.
Under the non-rotating, flat Earth assumption, the robot's state can be represented by the position $\bm{p}\in\R^3$, the velocity $\bm{v}\in\R^3$, and the orientation $\mathbf{R}\in\SO(3)$, with the noise-free dynamics given by
\begin{subequations}
    \begin{align}
        & \dot{\mathbf{R}} 
        =\mathbf{R} \left(\bm{\omega}- \bm{b}_{\bm{\omega}}\right)^{\wedge}, \\
        &  \dot{\bm{v}}= \mathbf{R}\left( \bm{a}- \bm{b}_{\bm{a}}\right)+ \bm{g}, \\
        &  \dot{\bm{p}}= \bm{v}, \\
        &  \dot{\bm{b}}_{\bm{\omega}}= \bm{\tau}_{\bm{\omega}}, \\
        &  \dot{\bm{b}}_{\bm{a}}= \bm{\tau}_{\bm{a}} ,
    \end{align}
    \label{eq:INS_dynamics}
\end{subequations}
Here $\bm{\omega}\in\R^3$ and $\bm{a}\in\R^3$ are the angular velocity and proper acceleration expressed in the body frame. 
The gravity vector is denoted by $\bm{g}\in\R^3$ and is expressed in the world frame.
The rigid body orientation, position and velocity $(\mathbf{R},\bm{p},\bm{v})$ are termed the \emph{navigation states}, expressed in the reference or world frame, denoted by $(\mathbf{R},\bm{p},\bm{v})\in\SO(3)\times\R^3\times\R^3$.
The two biases $\bm{b}_{\bm{\omega}}\in\R^3$ and $\bm{b}_{\bm{a}}\in\R^3$ are termed the \emph{bias states}, expressed in the body frame, with dynamics modelled by unknown inputs $\bm{\tau}_{\bm{\omega}}$ and $\bm{\tau}_{\bm{a}}$ sampled from Gaussian distributions.

In the IEKF framework, the navigation states are usually represented using the extended pose group $\bm{F} = (\mathbf{R},\bm{p},\allowbreak \bm{v})\in\SE_2(3)$ (where $\bm{F}$ stands for a Galilean frame), and the bias states are represented as $\bm{b} = (\bm{b}_{\bm{\omega}}, \bm{b}_{\bm{a}})\in\R^6$.
The IEKF state space is then given by the product group $\grpG = \SE_2(3)\times\R^6$, and we write an element of the state space as
\[
X = (\bm{F},\bm{b}) = (\mathbf{R},\bm{p},\bm{v},\bm{b}_{\bm{\omega}},\bm{b}_{\bm{a}})\in\grpG.
\]
The GNSS receiver provides the global position measurement.
The noise-free measurement model is given by
\begin{align}
    h(X) = \bm{p} + \eta \in\R^3.
\end{align}
with $\eta \sim \GP(0,R)$.

\begin{figure*}[t]
    \centering
    \includegraphics[width=\linewidth]{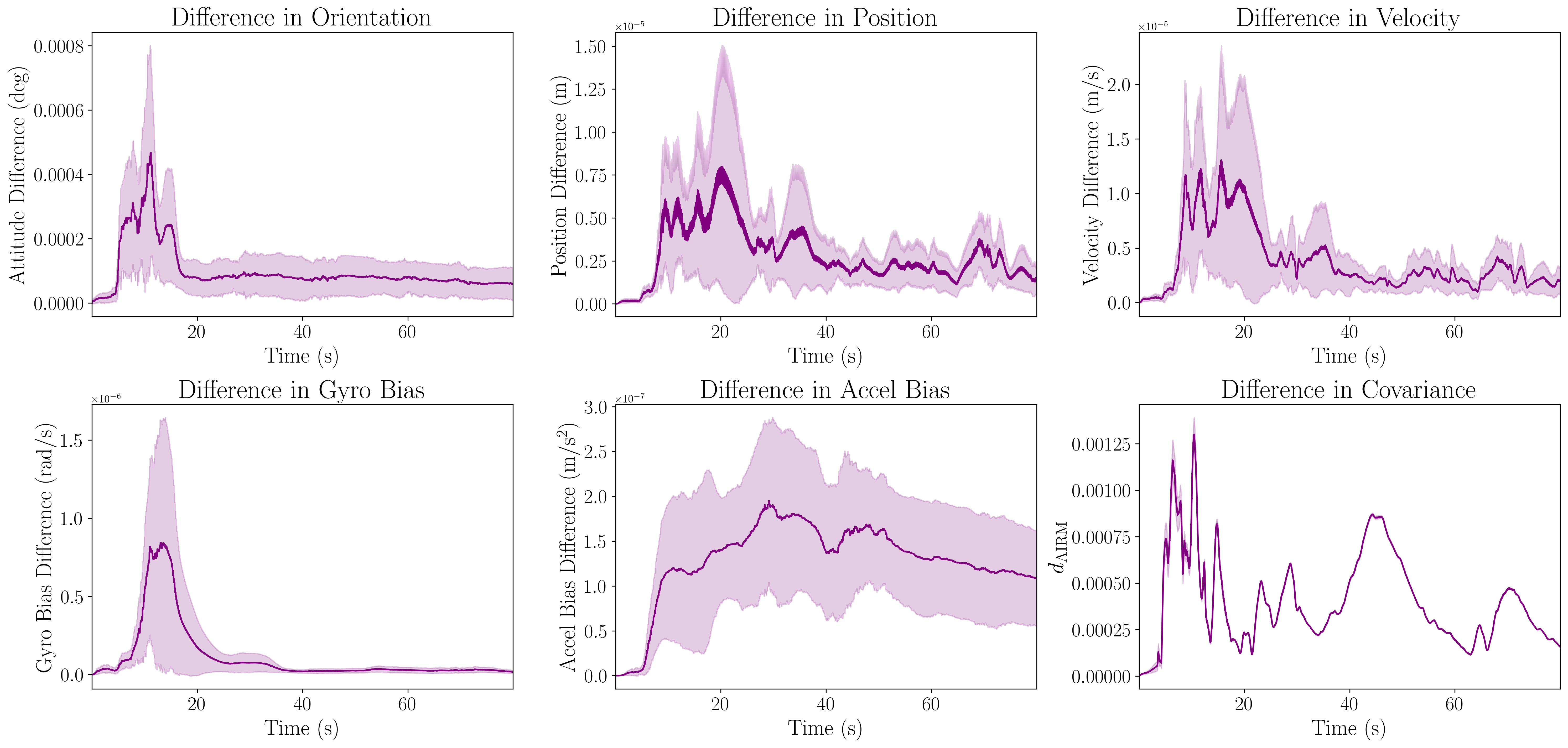}
    \caption{The difference in the estimates of each state variable and the covariance matrices between the L-IEKF and R-IEKF. The shaded area represents the standard deviation of the data across all trials.}
    \label{fig:equivalence_test_diff}
\end{figure*}

\subsection{Experimental Setup}
In this experiment, we conducted a Monte-Carlo simulation with 600 trials of a simulated UAV equipped with an IMU and a GNSS receiver.
We use the initial 80 seconds from six sequences of the Euroc dataset \cite{burriEuRoCMicroAerial2016} as the reference trajectories and generate 100 runs for each sequence.
Each run is initialised with a random state sampled from a Gaussian distribution with $20^\circ$ standard deviation per axis for the orientation and $1$m for the position.
The IMU biases are randomly generated every run following a Gaussian distribution with $0.1 \text{rad/s}\sqrt{s}$ for the gyroscope and $0.1 \text{m/s}^2\sqrt{s}$ for the accelerometer.
The global position measurements are corrupted by additive Gaussian noise with $0.2$m standard deviation per axis.
The UAV receives GNSS measurements at a rate of 10Hz, while the IMU measurements are sampled at 200Hz.
This experiment setup is the same as the one used in \cite{fornasierEquivariantSymmetriesInertial2025}.

The initial covariance matrices for both filters are set to be the initial Gaussian distribution which is used to sample the initial state, converted to a L-CGD or R-CGD representation on $\SE_2(3)$, that is, the initial states of two filters satisfy the equivalence condition (Def.~\ref{def:equivalent}).
In all comparisons, identical input data was provided to both filters. 

\begin{figure*}
    \centering
    \includegraphics[width=\linewidth]{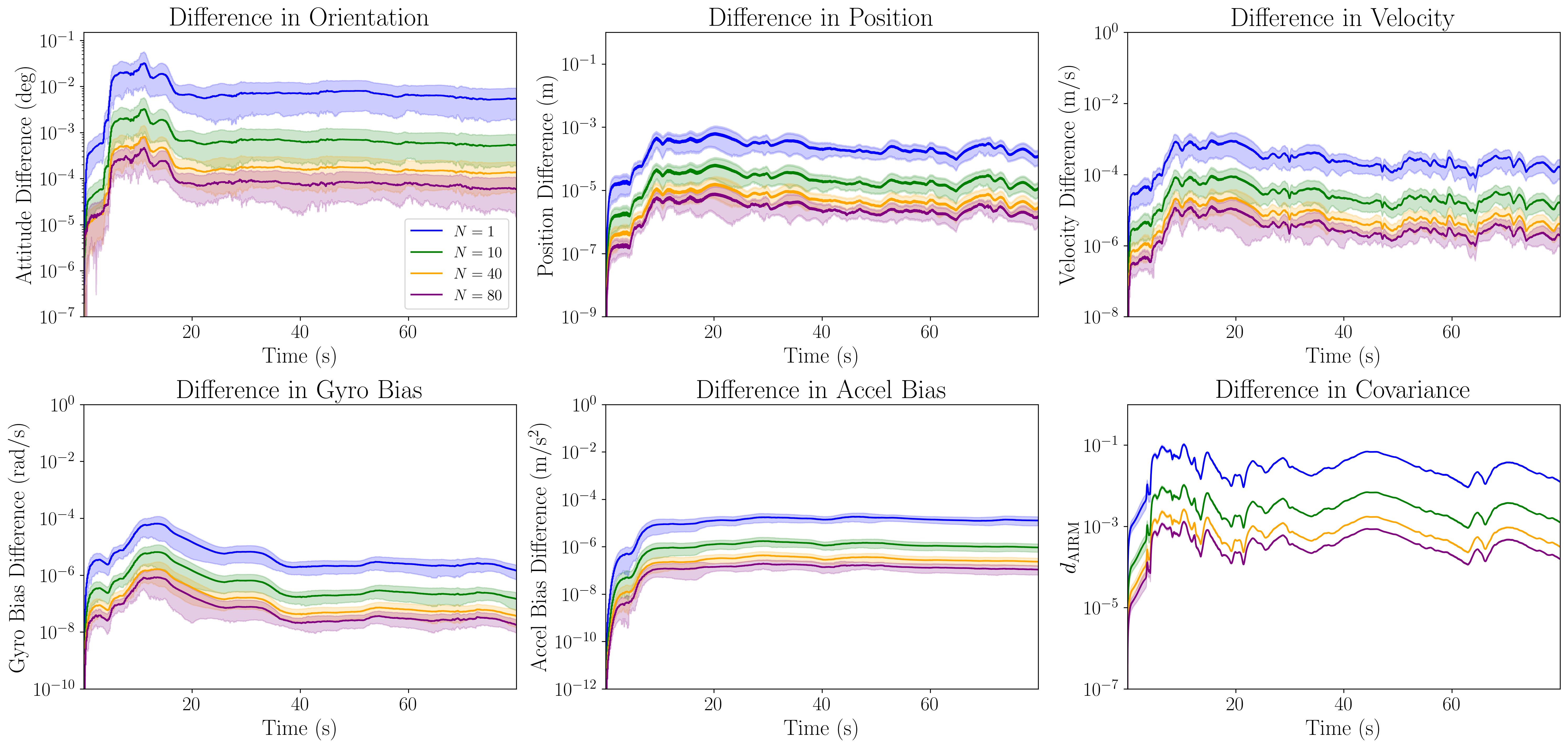}
    \caption{The difference between L-IEKF and R-IEKF with different $N$-substep Euler counts in the prediction step. 
    Shaded areas represent the standard deviation of the data across all trials.
    Log scale is used for better visibility of the difference.}
    \label{fig:frequency_test}
\end{figure*}

\subsection{Filter Equivalence}

In \ref{app:discrete_equivalence}, we show that for the natural discrete time system \eqref{eq:dynamicsDT} the algebraic expressions for the L-IEKF and R-IEKF with reset are identical. 
We do not simulate results for this case as the updates are equal up to machine precision. 

In practice, a hybrid filter is implemented by integrating the continuous-time filter propagation ODEs with the IMU inputs until the GPS measurement is received, and then applying the update step with the GPS measurement.
As the result, it is necessary to numerically approximate the solution of the continuous-time system.
Unfortunately, the standard integration schemes (e.g., Euler or Runge-Kutta) are not equivalent between the L-IEKF and R-IEKF.
In continuous-time, the L-IEKF and R-IEKF are related by a time-varying change of coordinates through $\Ad_{\hat{X}(t)}$.
Numerical integration of a time-varying system is not invariant to time-varying changes of coordinates, and this introduces a discrepancy between the two solutions.

For the scenario considered, the IMU provides 20 samples for every GPS sample received.
In practice, even integrating at the IMU rate, the numerical approximation introduces noticeable (if small) variation between the left and right variants of the filter.
To characterise and understand the discretisation error, we implement a sub-stepping scheme along with the standard explicit Euler integration in the prediction step: the continuous-time filter dynamics between two IMU measurement instants are integrated at a higher internal rate with zero-order hold inputs.
Concretely, we subdivide the IMU sample interval into $N$ equal sub-steps and apply the IEKF prediction operation $N$ times while keeping the inertial input constant over each sub-step.
Note that this sub-stepping scheme is purely a numerical integration refinement that captures the time varying nature of the equivalence relationship $\Ad_{\hat{X}(t)}$ along the trajectory. 
As $N$ goes to infinity, the sub-stepping solution converges to the true continuous time solution of the sampled data system with zero-order hold for the inputs and provides a framework to evaluate the contributions of the paper.

Fig.~\ref{fig:equivalence_test} provides results for the two filter implementations with a subsample rate of $N = 80$, that is 80 steps for every IMU measurement. 
The RMSE values for the navigation states and bias states, as well as the average normalised estimation error squared (ANEES) \cite{liEvaluationEstimationAlgorithms2012} are computed.
Both filters converge smoothly to the true state as shown in Fig.~\ref{fig:equivalence_test}, with the ANEES values close to 1, indicating that the filters are well-calibrated.
It is clear that the RMSE values of the navigation states and bias states for the L-IEKF and R-IEKF with reset are nearly identical.

In Fig.~\ref{fig:equivalence_test_diff}, we plot the difference in the state estimates of the L-IEKF and R-IEKF with reset for each state variable.
The difference of the covariance matrices is computed using the affine-invariant Riemannian distance (AIRM) of symmetric positive-definite matrices \cite{bhatiaPositiveDefiniteMatrices2007} after applying the adjoint map, given by
\[
d_{\text{AIRM}}(\Sigma_\tL, \Sigma_\tR') = \| \log(\Sigma_\tL^{-1/2} \Sigma_\tR' \Sigma_\tL^{-1/2}) \|_F,
\]
where $\Sigma_\tR' = \Ad_{\hat{x}_\tR^{-1}}^\vee \Sigma_\tR {\Ad_{\hat{x}_\tR^{-1}}^\vee}^\top$.
The value of AIRM represents, in percentage terms, how much the standard deviation of two covariance error-ellipsoids differs along each principal axis.

As shown in Fig.~\ref{fig:equivalence_test_diff}, a minor discrepancy in the rotation estimate is observed initially (approximately $0.0003\,\mathrm{deg}$), which decreases and remains around $0.0001\,\mathrm{deg}$ throughout the remainder of the trajectory.
The rest of the state variables, including position, velocity, and biases, show no significant difference between the two filters.
The AIRM value of the covariance matrices remains at around $0.0004$ throughout the trajectory, which corresponds to $0.04\%$ difference between the standard deviations of the two covariance matrices along each principal axis.

These results demonstrate the equivalence of the L-IEKF and R-IEKF with reset in practice.

\begin{figure*}[ht]
    \centering
    \includegraphics[width=\linewidth]{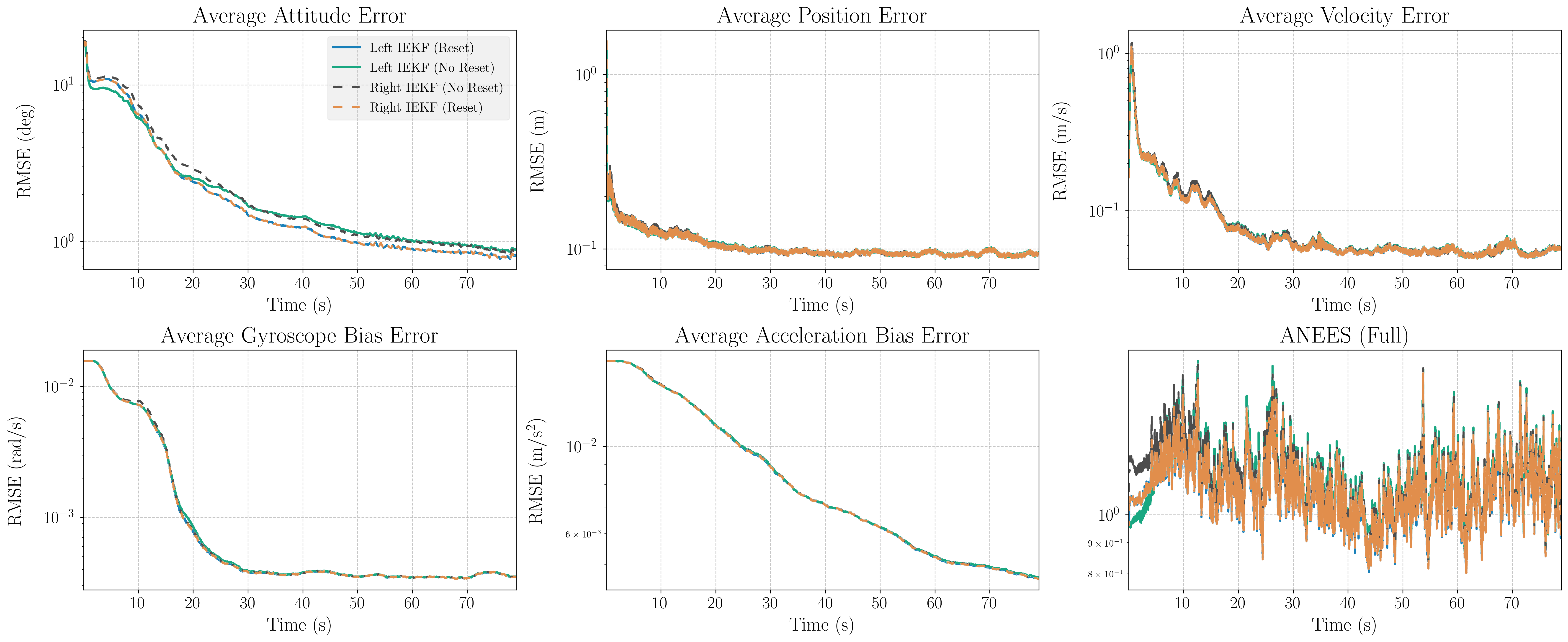}
    \caption{The RMSE and ANEES of the L-IEKF with reset (\textcolor{blue_okabe}{$\rule[0.5ex]{0.5cm}{1.0pt}$}) and without reset (\textcolor{green_okabe}{$\rule[0.5ex]{0.5cm}{1.0pt}$}), R-IEKF with reset (\textcolor{vermillion}{$\rule[0.5ex]{0.2cm}{1.0pt}\,\rule[0.5ex]{0.2cm}{1.0pt}$}) and without reset (\textcolor{black}{$\rule[0.5ex]{0.2cm}{1.0pt}\,\rule[0.5ex]{0.2cm}{1.0pt}$}).}
    \label{fig:reset_step_ablation}
\end{figure*}

\subsection{Discretisation Analysis}

In this section, we study the error introduced by the discretisation of the IEKF filter equation. 
In this experiment, we implement the L-IEKF and R-IEKF with different sub-stepping rates $N$ in the prediction step, differing from 1 to 80.
The results are shown in Fig.~\ref{fig:frequency_test}, with a zoomed-in view of the rotation difference in the inset to highlight the minor differences. 
Note that for $N =1$ we are still propagating for every IMU measurement and only applying the reset step (every 20 samples) when a GPS measurement is available.

It can be seen that when performing an explicit Euler integration with $N=1$, the average difference in the rotation estimate is at the scale of $0.01\,\mathrm{deg}$. 
As the sub-stepping rate increases, the differences in the state estimates and covariance matrices between the L-IEKF and R-IEKF decrease almost linearly.
Increasing the substep counts makes the discretised solution close to the true continuous solution and brings the trajectories together.

We claim that these results demonstrate that while discretisation introduces minor numerical differences between the implementations, these differences decrease with higher sampling frequencies and do not reflect any fundamental distinction between the filters.

\subsection{Reset Step Ablation Study}
\label{sec:reset_experiment}
The reset step is a key component in the equivalence proof. 
The original papers on the IEKF \cite{barrauInvariantExtendedKalman2017} and TFG-IEKF \cite{barrauGeometryNavigationProblems2023} do not include the reset step, relying instead on the underlying parallelism of a Lie group for the reset. 
Without the reset step, then the Left- and Right- variants of the IEKF are certainly not equivalent, and it is of interest to consider what difference in performance is present in this case. 
It is also of interest to ask whether the reset step improves the performance of the filter in general. 
In this section, we conduct an ablation study to investigate the effect of the reset step on the performance of the IEKF.

We implement the L-IEKF and R-IEKF with and without the reset step. 
Results are shown in Fig.~\ref{fig:reset_step_ablation}.
Broadly, the filters have similar responses and asymptotic performance, however, there are several important differences in the details. 
Consider the top left plot Fig.~\ref{fig:reset_step_ablation} showing rotational error. 
Here, the L-IEKF (the filter that matches the error symmetry to the measurement symmetry as is often recommended in the literature) demonstrates a significant improvement in performance during the transient (the first 15s of the response). 
The two algorithms with reset have identical responses (as expected) while the Right-IEKF (that opposes the error and measurement symmetry) performs worse. 
These results have been consistently replicated for a range of different scenarios and data and the authors believe that this reflects a substantive behaviour, although we only have empirical evidence for this. 
After the initial transient period, however, the algorithms that implement the reset step benefit from the improved stochastic properties of the modified covariance update. 
This effect has been noted in prior studies \cite{muellerCovarianceCorrectionStep2017, markleyErrorCovarianceResetMultiplicative2023} as well.
The observations that are most striking in the attitude are replicated to a lesser degree across the other state variables, including position and velocity. 
Similar trends are observed in the ANEES values.

The authors hypothesise that the initial performance advantage for the L-IEKF in the transient is related to the fact that the measurement noise in GPS lies in a linear space $\R^3$ and the Bayesian fusion step of the EKF is globally exact on Euclidean space when the error and measurement symmetries are matched. 
The reset step introduces coupling and nonlinearities into this update step due to the remapping of the covariance.  
We believe that during the transient when the error is very large, the benefit of the globally exact fusion outweighs the benefit of the stochastic correction of the reset. 
However, during the asymptotic phase, the Bayesian fusion in local coordinates does not incur the same performance loss and the benefit of the reset step dominates. 
We have not obtained empirical or analytic evidence to support this hypothesis at this time. 

\begin{remark}
The INS example chosen only considers a left-invariant measurement model, while the equivalence of the L-IEKF and R-IEKF with reset holds more generally. 
The specific performance characteristics of filters may vary depending on the measurement model, for example the navigation problem with body-frame measurements can lead to unobservability issues.
The implication of the reset step in such scenarios is an interesting topic for future research.
\end{remark}

\section{Conclusion}
\label{sec:conclusion}
In this paper, we have demonstrated the mathematical and practical equivalence of the left-Invariant extended Kalman filter (L-IEKF) and the right-Invariant extended Kalman filter (R-IEKF) with reset.
Our theoretical analysis proves that when properly initialised and implemented with the correct reset step, the two filters yield identical state estimates and equivalent probability distributions regardless of the choice of handedness.
The experimental results using a GNSS-aided inertial navigation system confirm the theoretical findings.
Furthermore, we show that for the INS problem, an L-IEKF without reset may well outperform the IEKF implemented with reset step during the transient, but not during the asymptotic phase of the filter. 
This effect may well underlie the community's understanding that matching the filter symmetry to the measurement is important, however, we recommend that the filter with reset is used as the default and then the choice of handedness is moot.

\appendix
\section{Equivalence of the L-IEKF and R-IEKF for discrete-time systems}\label{app:discrete_equivalence}

In the main result (Sec.~\ref{sub:equivalence}), we proved that the prediction step of the L-IEKF and R-IEKF yields equivalent estimates for hybrid systems.
However, the equivalence no longer holds when the continuous-time filter dynamics are discretised and implemented on a digital computer.
In this section, we show that given a discrete-time system, the discrete L-IEKF and R-IEKF with reset are equivalent as well.
We will only consider the prediction step of the L-IEKF and R-IEKF, as the update and reset steps remain identical in the discrete-time case.

Consider a noise-free discrete-time system on a Lie group $\grpG$ given by
\begin{align*}
    X_{k+1} = F_v(X_k) = X_k \discreteLift{X_k, v_k},
\end{align*}
where $X_k\in\grpG$ is the true state at time step $k$, $F_v(X_k)$ is the discrete-time flow induced by the velocity $v_k\in\vecL$.
One may define the system evolution function with left trivialisation by $\discreteLift{X_k, v_k} = X_{k+1} X_k^{-1}\in\grpG$.
We make the assumption that the process noise is left-invariant and enters the system as a Gaussian process on the Lie algebra $\gothg$, i.e.
\begin{align}
    X_{k+1} = X_k \discreteLift{X_k, v_k} \expG(w_k^\wedge), \quad w_k \sim \GP(0, Q_k)
\label{eq:dynamicsDT}
\end{align}
where $Q_k$ is a positive-definite covariance matrix.
This follows the definition of the general dynamical system provided in \cite{barrauInvariantKalmanFiltering2018}.
To enhance readability and avoid excessive use of subscripts, we will drop the subscript $k, k+1$ and use the notation $X, X^-$ to align with the notation used in the main text.

\begin{remark} 
If one considers a pure left-invariant continuous-time system with noise \[\td {X} = X (\Lambda(v)\dt+\Upsilon[Q^{\frac{1}{2}}dw])\] then if the velocity $v = v_k$ is constant on short time intervals, the model can be integrated into an exact noisy discrete-time model on the Lie-group 
\begin{align}\label{eq:system_dt}    
X_{k+1} = X_{k} \expG(\Lambda(v_k) \delta t + w_k^\wedge), \quad w_k \sim \GP(0, \delta t \Upsilon Q \Upsilon^\top)
\end{align}
where the Gaussian process lives on the Lie-algebra $\gothg$. 
This construction differs from the general discrete-time model \eqref{eq:dynamicsDT} by the higher-order terms in the Baker–Campbell– Hausdorff (BCH) expansion. 
One may choose to keep the first commutator from the BCH expansion to obtain second-order accuracy in the noise model, however, the modelling choice does not affect the equivalence proof. 
\end{remark}

Given the discrete-time system \eqref{eq:dynamicsDT} and the pre-observer system 
\[
\hat{X}^- = \hat{X} \discreteLift{\hat{X}, v},
\]
the left-invariant error and its dynamics are given by
\begin{align*}
    E_{\tL}^- &= {\hat{X}^{-}_\tL}^{-1} X^-_\tL \\
    &= \discreteLift{\hat{X}_\tL, v}^{-1} E_{\tL} \discreteLift{\hat{X}_\tL E_{\tL}, v}\expG(w^\wedge).
\end{align*}
Similarly, one can define the right-invariant error and its dynamics
\begin{align*}
    E_{\tR}^- &= X^-_\tR {\hat{X}^{-}_\tR}^{-1} \\
    &= E_{\tR} \hat{X}_\tR \discreteLift{E_{\tR} \hat{X}_\tR, v}\expG(w^\wedge)\discreteLift{\hat{X}_\tR, v}^{-1}\hat{X}_\tR^{-1} .
\end{align*}

\begin{definition}[L-IEKF discrete-time prediction step]
    Let $(\mu_\tL, \hat{X}_\tL, \Sigma_\tL)$ denote the L-IEKF state with $\mu_\tL =0$ at time $t_k$.
    The discrete-time prediction step of the L-IEKF is given by
    \begin{align*}
        \mu_\tL^- &= 0,\\
        \hat{X}_\tL^- &= \hat{X}_\tL \discreteLift{\hat{X}_\tL, v}, \\
        \Sigma_\tL^- &= \mathbf{A}_\tL \Sigma_\tL \mathbf{A}_\tL^\top + \mathbf{B}_\tL Q \mathbf{B}_\tL^\top, 
    \end{align*}
    where $\mathbf{A}_\tL$ and $\mathbf{B}_\tL$ are the linearisation matrices defined by 
    \begin{align*}
        \mathbf{A}_\tL &:= \tD \tL_{{\discreteLift{\hat{X}_\tL, v}}^{-1}} \cdot \tD_X|_{\hat{X}_\tL} \discreteLift{X, v} \cdot \tD \tL_{\hat{X}_\tL} \\
    &\hspace{5cm}+ \Ad^\vee_{{\discreteLift{\hat{X}_\tL, v}}^{-1}}, \\
    \mathbf{B}_\tL &:= I. 
    \end{align*}
    
\end{definition}

\begin{definition}
    [R-IEKF discrete-time prediction step]
    Let $(\mu_\tR, \hat{X}_\tR, \Sigma_\tR)$ denote the R-IEKF state with $\mu_\tR =0$ at time $t_k$.
    The discrete-time prediction step of the R-IEKF is given by
    \begin{align*}
        \mu_\tR^- &= 0, \\
        \hat{X}_\tR^- &= \hat{X}_\tR \discreteLift{\hat{X}_\tR, v}, \\
        \Sigma_\tR^- &= \mathbf{A}_\tR \Sigma_\tR \mathbf{A}_\tR^\top + \mathbf{B}_\tR Q \mathbf{B}_\tR^\top,
    \end{align*}
    where $\mathbf{A}_\tR$ and $\mathbf{B}_\tR$ are the linearisation matrices defined by 
    \begin{align*}
        \mathbf{A}_\tR &:= \Ad^\vee_{\hat{X}_\tR} \cdot \tD \tR_{\discreteLift{\hat{X}_\tR, v}} \cdot \tD_X|_{\hat{X}_\tR} \discreteLift{X, v} \cdot \tD \tR_{\hat{X}_\tR} - I ,\\
        \mathbf{B}_\tR &:= \Ad^\vee_{\hat{X}_\tR^-}.
    \end{align*}
\end{definition}

\begin{lemma}(Discrete-time prediction step equivalence)
    The discrete-time prediction steps of the L-IEKF and R-IEKF are equivalent, i.e. 
    Let $(\mu_\tL, \hat{X}_\tL, \Sigma_\tL)$ and $(\mu_\tR, \hat{X}_\tR, \Sigma_\tR)$ be the states of the L-IEKF and R-IEKF, respectively, with $\mu_\tL = 0$ and $\mu_\tR = 0$.
    Suppose that the L-IEKF and R-IEKF states are equivalent (Def. \ref{def:equivalent}) before the prediction step, then the states after the prediction step remain equivalent.
\end{lemma}
\begin{proof}
    It is straightforward to see that the mean $\mu$ and the reference state $\hat{X}$ remain equivalent after the prediction step.
    To prove the equivalence of the covariance matrices, it suffices to show that 
    \begin{align}
        \mathbf{A}_\tR &= \Ad^\vee_{\hat{X}_\tL^-} \mathbf{A}_\tL {\Ad^\vee_{\hat{X}_\tL}}^{-1}, \label{eq:AR_AL_equivalence_DT}\\
        \mathbf{B}_\tR &= \Ad^\vee_{\hat{X}_\tL^-} \mathbf{B}_\tL. \label{eq:BR_BL_equivalence_DT}
    \end{align}
    To prove \eqref{eq:AR_AL_equivalence_DT}, one has
    \begin{align*}
        &\Ad^\vee_{\hat{X}_\tL^-} \mathbf{A}_\tL {\Ad^\vee_{\hat{X}_\tL}}^{-1}\\
        &=\Ad^\vee_{\hat{X}_\tL^-} \tD\tL_{{\discreteLift{\hat{X}_\tL, v}}^{-1}} \tD_X|_{\hat{X}_\tL} \discreteLift{X, v} \tD \tL_{\hat{X}_\tL}{\Ad^\vee_{\hat{X}_\tL}}^{-1} \\
        &\hspace{3.5cm}+ \Ad^\vee_{\hat{X}_\tL^-} \Ad^\vee_{{\discreteLift{\hat{X}_\tL, v}}^{-1}} {\Ad^\vee_{\hat{X}_\tL}}^{-1} \\
        &=\tD\tL_{\hat{X}_\tL^-}\tD\tR_{\hat{X}_\tL^-}\tD_X|_{\hat{X}_\tL} \discreteLift{X, v} \tD\tL_{\hat{X}_\tL} \tD\tR_{\hat{X}_\tL} \tD\tR_{\hat{X}_\tL} \\
        &\hspace{3.5cm}+ \Ad^\vee_{\hat{X}_\tL^-} \Ad^\vee_{{\discreteLift{\hat{X}_\tL, v}}^{-1}} {\Ad^\vee_{\hat{X}_\tL}}^{-1} \\
        &= \Ad_{\hat{X}_\tL}^\vee \tD\tR_{\discreteLift{\hat{X}_\tL, v}} \tD_X|_{\hat{X}_\tL} \discreteLift{X, v} \tD\tR_{\hat{X}_\tL} - I \\
        &= \mathbf{A}_\tR.
    \end{align*}
    The proof of condition \eqref{eq:BR_BL_equivalence_DT} follows directly from the definition of $\mathbf{B}_\tR$ and $\mathbf{B}_\tL$.
    This completes the proof of the lemma.
\end{proof}

%%%%%%%%%%%%%%%%%%%%%%%%%%%%%%%%%%%%%%%%%%%%%%%%%%%
\bibliographystyle{IEEEtran}
\bibliography{reference}

\end{document}